\documentclass[conference]{IEEEtran}
\IEEEoverridecommandlockouts
\usepackage{cite}
\usepackage{amsmath,amssymb,amsfonts}
\usepackage{graphicx}
\usepackage{textcomp}
\usepackage{xcolor}
\usepackage{amsthm}
\usepackage{times}
\usepackage{xcolor}
\usepackage{soul}
\usepackage[small]{caption}
\usepackage{multicol}
\usepackage{subfigure}
\usepackage{enumerate}
\usepackage{amsmath,amssymb,amstext}
\usepackage{graphicx}
\usepackage{float}
\usepackage{algorithm}
\usepackage{multirow}
\usepackage{booktabs}
\usepackage{algpseudocode}
\usepackage{bm}
\usepackage{tabularx}

\newtheorem{theorem}{Theorem}

\newtheorem{lemma}[theorem]{Lemma}

\newtheorem{definition}{Definition}

\newtheorem{assumption}[theorem]{Assumption}

\,
\,

\newcommand{\be}{\begin{equation}}
\newcommand{\ee}{\end{equation}}
\newcommand{\ba}{\begin{array}}
	\newcommand{\ea}{\end{array}}
\newcommand{\bpm}{\begin{pmatrix}}
	\newcommand{\epm}{\end{pmatrix}}

\DeclareMathOperator*{\argmin}{arg\,min}
\newcommand{\etal}{\textit{et al.}}

\def\BibTeX{{\rm B\kern-.05em{\sc i\kern-.025em b}\kern-.08em
    T\kern-.1667em\lower.7ex\hbox{E}\kern-.125emX}}
\begin{document}
\bibliographystyle{IEEEtran}

\title{An Efficient ADMM-Based Algorithm to Nonconvex Penalized Support Vector Machines
}

\author{\IEEEauthorblockN{Lei Guan$^{\dag}$, Linbo Qiao$^{\dag}$, Dongsheng Li$^{\dag}$, Tao Sun$^{\ddag}$, Keshi Ge$^{\dag}$, Xicheng Lu$^{\dag}$}
\IEEEauthorblockA{\textit{$\dag$College of Computer}, \textit{$\ddag$Department of Mathematics} \\
\textit{National University of Defense Technology}\\
Changsha, China \\
\{guanleics,linboqiao\}@gmail.com,\, dsli@nudt.edu.cn,\, nudtsuntao@163.com,\, \{gekeshi,xclu\}@nudt.edu.cn}
}

\maketitle

\begin{abstract}
Support vector machines (SVMs) with sparsity-inducing nonconvex penalties have received considerable attentions for the characteristics of automatic classification and variable selection.
However, it is quite challenging to solve the nonconvex penalized SVMs due to their nondifferentiability, nonsmoothness and nonconvexity.
In this paper, we propose an efficient ADMM-based algorithm to the nonconvex penalized SVMs.
The proposed algorithm covers a large class of commonly used nonconvex regularization terms including the smooth clipped absolute deviation (SCAD) penalty, minimax concave penalty (MCP), log-sum penalty (LSP) and capped-$\ell_1$ penalty.
The computational complexity analysis shows that the proposed algorithm enjoys low computational cost.
Moreover, the convergence  of the proposed algorithm is guaranteed.
Extensive experimental evaluations on five benchmark datasets demonstrate the superior performance of the proposed algorithm to other three state-of-the-art approaches.
\end{abstract}


\section{Introduction}


It is well-known that SVMs can perform automatic variable selection by adding a sparsity-inducing penalty (regularizer) to the loss function \cite{zhu20041,zhangvariable}. Typically, the sparsity-inducing penalties can be divided into two catagories: convex penalty and nonconvex penalty. The $\ell_1$ penalty is the most famous convex penalty and has been widely used for variable selection \cite{tibshirani1996regression,zhu20041}.
Commonly used nonconvex penalties include $\ell_p$ penalty with $0<p<1$, smooth clipped absolute deviation (SCAD) penalty \cite{fan2001variable}, log penalty \cite{mazumder2011sparsenet}, minimax concave penalty (MCP) \cite{zhang2010nearly}, log-sum penalty (LSP) \cite{candes2008enhancing}, and capped-$\ell_1$ penalty \cite{zhang2010analysis}. It has been shown in literatures that nonconvex penalties outperform  the convex  ones with  better statistics properties \cite{yao2015fast};
theoretically, SVMs with elaborately designed nonconvex penalties can asymptotically unbiasedly estimate the large nonzero parameters as well as shrink the estimates of zero valued parameters to zero\cite{Liu2010Sparse}.
Consequently, the nonconvex penalized SVMs conduct variable selection and classification simultaneously; and they are more robust to the outliers and are able to yield a compact classifier with high accuracy.

Although nonconvex penalized SVMs are appealing, it is rather hard to optimize due to the nondifferentiability of the hinge loss function and the nonconvexity and nonsmoothness introduced by the nonconvex regularization term. Existing solutions to nonconvex penalized SVMs \cite{zhang2005gene,zhangvariable} are pretty computationally inefficient, and they are limited to a few of nonconvex penalties. Besides that, other popular approaches are unable to apply to the nonconvex penalized hinge loss function since they typically require the loss function to be differentiable \cite{gong2013a,jiang2016structured}.

In this paper, we focus on solving the standard support vector machines with a general class of nonconvex penalties including the SCAD penalty, MCP, LSP and capped-$\ell_1$ penalty.
Mathematically, given a train set $\mathcal{S}=({\bf x}_i, y_i)_{i=1}^n$, where ${\bf x}_i \in\mathbb{R}^{d}$ and $y_i\in\{-1,1\}$,
 the nonconvex penalized SVMs minimize the following penalized hinge loss function:
\begin{equation}
\label{eq1}
\min_{\{{\bf w},b\}} \ \ \frac{1}{n}\sum_{i=1}^{n}[1-y_i({\bf w}^\top {\bf x}_i+b)]_++P({\bf w}),
\end{equation}
where the $\{{\bf w},b\}$ pair is the decision variable with ${\bf w}\in\mathbb{R}^{d}$ and $b\in\mathbb{R}$.
$P({\bf w})=\sum_{j=1}^{d}p_\lambda(w_j)$ is the regularization term with a tunning parameter $\lambda$, and $p_\lambda(w_j)$ is one of the nonconvex regularizers listed in Table \ref{Tab:Penalty}. Here and throughout this paper, $[a]_+$ represents $\text{max}(a,0)$; and $(\cdot)^\top$ denotes the transposition of matrices and column vectors.

\begin{table}[htbp]
	\centering
	\caption{Example nonconvex regularizers. Here, $\theta>2$ for SCAD regularizer and $\theta>0$ for other regularizers.
	}
	\label{Tab:Penalty}
	\begin{tabular}{|l|l|} \hline
		Name & $p_{\lambda}(w_j)$ \\ \hline
		LSP & $\lambda\log\left(1+ |w_j|/\theta\right) $ \\ \hline
		SCAD &  $\left\{\begin{array}{ll}
		\lambda|w_j|, & \text{if} \ |w_j|\leq\lambda, \\
		\frac{-w_j^2+2\theta\lambda|w_j|-\lambda^2}{2(\theta-1)}, & \text{if} \ \lambda<|w_j|\leq\theta\lambda, \\
		\frac{(\theta+1)\lambda^2}{2}, & \text{if} \ |w_j|>\theta\lambda. \\
		\end{array}\right.$ \\ \hline
		MCP & $\left\{\begin{array}{ll}
		\lambda|w_j| - w_j^2/(2\theta), & \text{if} \ |w_j|\leq\theta\lambda, \\
		\theta\lambda^2/2, & \text{if} \ |w_j|>\theta\lambda. \\
		\end{array}\right.$ \\ \hline
		Capped-$\ell_1$ & $\lambda\min\left( |w_j|, \theta\right) $ \\ \hline
	\end{tabular}
\end{table}

To address problem \eqref{eq1}, we propose an efficient algorithm by incorporating the framework of
 alternating direction method of multipliers (ADMM) \cite{boyd2011distributed}. The main contributions of this paper can be summarized as follows.
 \begin{enumerate}[a)]
 \item We find that by reasonably reformulating problem \eqref{eq1} and applying the framework of ADMM, nonconvex penalized SVMs can be solved by optimizing a series of subproblems. In addition, each subproblem owns  a closed-form solution and is easy to solve.
\item More importantly, we find the main computational burden of the ADMM procedure lies in the update of ${\bf w}$ which requires to calculate the inversion of a $d \times d$ matrix.
It costs $O(d^3)$ flops (floating point operations) when the order of $d$ is bigger than $n$.
We propose an efficient scheme to update ${\bf w}$ via using the Sherman-Morrison formula \cite{sherman1950adjustment} and Cholesky factorization, attaining an improvement by a factor of $O(d/n)^2$ over the naive method in this case. Furthermore, we optimize the iteration scheme so that the computationally expensive part is calculated only once.
\item We present detailed computational complexity analysis and show that the optimized algorithm is pretty computationally efficient.
\item In addition, we also present detailed convergence analysis of the proposed ADMM algorithm.
\item Extensive experimental evaluations on five LIBSVM benchmark datasets demonstrate the outstanding performance of the proposed algorithm. In comparison with other three state-of-the-art algorithms, the proposed algorithm runs faster as well as attains high prediction accuracy.
\end{enumerate}

The rest of this paper is organized as follows. Section \ref{relatedwork} reviews the related work. Section \ref{algorithm} presents the derivation procedure and studies the computational complexity of the proposed algorithm. Section \ref{convergenceanalysis} shows the convergence analysis. Section \ref{experimentalresults} details and discusses the experimental results. Finally, we conclude this paper in Section \ref{conclusions}.

\section{Related Work}\label{relatedwork}
Lots of studies have been devoted to  the nonconvex penalized SVMs due to their superior performance in various applications arising from academic community and industry.
Liu \etal \cite{Liu2010Sparse} developed an $\ell_p$-norm penalized SVM with nonconvex penalty $\ell_p$-norm ($0<p<1$) based on margin maximization and $\ell_p$ approximation.
Zhang \etal \cite{zhang2005gene} combined SVM with smoothly clipped absolute deviation (SCAD) penalty, and obtained a compact classifier with high accuracy.
In order to efficiently solve SCAD-penalized SVM, they proposed  a successive quadratic algorithm (SQA) which converted the non-differentiable and non-convex optimization problem into an easily solved linear equation system.
Zhang \etal \cite{zhangvariable} established a unified theory for SCAD- and MCP-penalized SVM in the high-dimensional setting.
Laporte \etal \cite{laporte2014nonconvex} proposed a general framework for feature selection in learning to rank using SVM with nonconvex regularizations such as log penalty, MCP and $\ell_p$ norm with $0<p<1$.
Recently, Zhang \etal \cite{Zhang2016Variable} have established an unified theory for a general class of nonconvex penalized SVMs in the high-dimensional setting.
Liu \etal \cite{liu2016global} showed that a class of nonconvex learning problems are equivalent to general quadratic programmings and proposed a reformulation-base technique named mixed integer programming-based global optimization (MIPGO).

Apart from previous work discussed above,  many researches concerned with optimization problems with a general class of nonconvex regularizations \cite{gong2013a,hong2016convergence,wang2015global,jiang2016structured} are developed.
Nevertheless, these proposed methods cannot be applied to solve the optimization problem studied in this paper.
In \cite{hong2016convergence}, Hong \etal\quad analyzed the convergence of the ADMM for solving certain nonconvex consensus and sharing problems. However, they require the nonconvex regularization term to be smooth, which violates the nonsmooth trait of the penalty functions considered in this paper.
Later, Wang \etal \cite{wang2015global} analyzed the convergence of ADMM for minimizing a nonconvex and possibly nonsmooth objective function with coupled linear  constraints.
Unfortunately, their analysis cannot be applied to the nonconvex penalized hinge loss function since they require the objective to be differentiable.
Gong \etal \cite{gong2013a} proposed General Iterative Shrinkage and Thresholding (GIST) algorithm to solve the nonconvex optimization problem for a large class of nonconvex penalties.
Recently, Jiang \etal \cite{jiang2016structured} have proposed two proximal-type variants of the ADMM to solve the structured nonconvex and nonsmooth problems.
Nevertheless, the algorithms proposed in \cite{gong2013a} and \cite{jiang2016structured} are unable to solve the nonconvex penalized hinge loss function because they both require the loss function to be differentiable as well.

\section{Algorithm For Nonconvex Penalized SVMs}\label{algorithm}
In this section, we derive the solution of nonconvex penalized SVMs based on the framework of  ADMM \cite{boyd2011distributed}.
By introducing auxiliary variables and reformulating the original optimization problem, the nonconvex penalized SVMs can be solved by iterating a series of subproblems with closed-form solutions. Moreover, detailed computational complexity analysis of the proposed algorithm is presented in this section.

\subsection{Derivation Procedure}\label{AA}
In order to apply the framework of ADMM, we first introduce auxiliary variables to handl the nondifferentiability of problem (\ref{eq1}).

Let ${\bf X}=[{\bf x}_1,\cdots,{\bf x}_n]^\top$ (${\bf X}\in\mathbb{R}^{n\times d}$) and ${\bf y}=[y_1,\cdots,y_n]^\top$.
Then the unconstrained problem~\eqref{eq1} can be rewritten as an equivalent form
\begin{equation}
\label{eq2}
\begin{array}{cl}
\displaystyle\min_{\{{\bf w},b,\bm{\xi}\}} & \frac{1}{n}{\bf 1}^\top\bm{\xi} + P({\bf w}), \\
\textnormal{s.t.}& {\bf Y}({\bf Xw}+b{\bf 1}) \succeq {\bf 1}-\bm{\xi}, \\
& \bm{\xi} \succeq {\bf 0},  \\
\end{array}
\end{equation}
where $\bm{\xi}=(\xi_1,\cdots,\xi_n)^\top$ and ${\bf Y}$ is a $n\times n$ diagonal matrix with $y_i$ on the $i$th diagonal element, i.e., ${\bf Y}=\text{diag}\{\bf y\}$. In what follows, ${\bf 1}$ is an $n$-column vector of $1$s, ${\bf 0}$ is an $n$-column vector of $0$s, and $\succeq$ denotes element-wised $\geqq$.

Note that, using variable splitting
and introducing another slack variable ${\bf s}$,
 problem~\eqref{eq2} can be converted to following equivalent constrained problem:
\begin{equation}
\label{eq3}
\begin{array}{cl}
\displaystyle\min_{\{{\bf w},b,\bm{\xi},\bm{s},\bm{z}\}} & \frac{1}{n}{\bf 1}^\top\bm{\xi} + P({\bf z}), \\
\textnormal{s.t.}&{\bf w}={\bf z}, \\
&  {\bf Y}({\bf Xw}+b{\bf 1}) + {\bm \xi} = {\bf s}+{\bf 1},\\
& \bm{\xi} \succeq {\bf 0},{\bf s} \succeq {\bf 0},
\end{array}
\end{equation}
where ${\bf z}=(z_1,\cdots,z_d)^\top$ and ${\bf s}=(s_1,\cdots,s_n)^\top$.

Hence, the corresponding surrogate Lagrangian function of~\eqref{eq3} is
\begin{equation}
\label{eqLag}
\begin{array}{ccl}
&&\mathcal{L}_0({\bf w},b,{\bf z},\bm{\xi},{\bf s},{\bm \gamma},{\bm \tau})\\
&=& \frac{1}{n}\textbf{1}^\top\bm{\xi} + P_\lambda({\bf z})+ <{\bm \gamma},({\bf w}-{\bf z})>  \\
&& +<{\bm \tau}, {\bf Y}({\bf X}{\bf w}+b{\bf 1})+\bm{\xi}-{\bf s}-{\bf 1}>,
\end{array}
\end{equation}
where ${\bm \gamma}\in\mathbb{R}^{d}$ and ${\bm \tau}\in\mathbb{R}^{n}$ are the dual variables corresponding to the first and second linear constraints of~\eqref{eq3}, respectively. $<\cdot,\cdot>$ represents the standard inner product in Euclidean space.
Note that we call $\mathcal{L}_0({\bf w},b,{\bf z},\bm{\xi},{\bf s},{\bf u},{\bf v})$ as ``surrogate Lagrangian function" since it does not involve the set of constraints $\{\bm{\xi} \succeq {\bf 0}, {\bf s} \succeq {\bf 0}\}$.
The projections onto these two simple linear constraints can be easily calculated by basic algebra computations and projections to the  1-dimensional  nonnegative set ($[0,+\infty)$).

Let ${\bf H}={\bf YX}$ and note that ${\bf y}={\bf Y}{\bf 1}$, thus the scaled-form surrogate augmented Lagrangian function can be written as
\begin{equation}
\label{eq4}
\begin{array}{ccl}
&&\mathcal{L}({\bf w},b,{\bf z},\bm{\xi},{\bf s},{\bf u},{\bf v})\\
&=&\mathcal{L}_0({\bf w},b,{\bf z},\bm{\xi},{\bf s},{\bm \gamma},{\bm \tau})+\frac{\rho_1}{2}||{\bf w-z}||_2^2  \\
&&+ \frac{\rho_2}{2}||{\bf Hw}+b{\bf y}+\bm{\xi}-{\bf s}-{\bf 1}||_2^2
\\
&=& \frac{1}{n}\textbf{1}^\top\bm{\xi} + P({\bf z}) + \frac{\rho_1}{2}||{\bf w-z+u}||_2^2 \\
&&+ \frac{\rho_2}{2}||{\bf Hw}+b{\bf y}+\bm{\xi}-{\bf s}-{\bf 1}+{\bf v}||_2^2+\text{constant},
\end{array}
\end{equation}
where ${\bf u}={\bm \gamma}/\rho_1$ and ${\bf v}={\bm \tau}/\rho_2$ are the scaled dual variables. The constants $\rho_1$ and $\rho_2$ are penalty parameters with $\rho_1>0$ and $\rho_2>0$.

The resulting ADMM procedure starts with an iterate ${\bf w}^{(0)},b^{(0)},{\bf z}^{(0)},\bm{\xi}^{(0)},{\bf s}^{(0)},{\bf u}^{(0)},{\bf v}^{(0)}$; and for each iteration count $k=0,1,2,\cdots$, the scaled-form ADMM procedure can be expressed as
\begin{align}
&{\bf w}^{(k+1)}=\argmin_{\bf w}\ \mathcal{L}({\bf w}, b^{(k)}, {\bf z}^{(k)},  \bm{\xi}^{(k)}, {\bf s}^{(k)}, {\bf u}^{(k)}, {\bf v}^{(k)}),\label{eq5}\\
&b^{(k+1)}=\argmin_b\ \mathcal{L}({\bf w}^{(k+1)}, b, {\bf z}^{(k)}, \bm{\xi}^{(k)}, {\bf s}^{(k)}, {\bf u}^{(k)}, {\bf v}^{(k)}),\label{eq6}\\
&{\bf z}^{(k+1)}=\argmin_{\bf z}\ \mathcal{L}({\bf w}^{(k+1)}, b^{(k+1)}, {\bf z}, \bm{\xi}^{(k)}, {\bf s}^{(k)}, {\bf u}^{(k)}, {\bf v}^{(k)})\nonumber\\
\label{eq7}\\
&\bm{\xi}^{(k+1)}=\argmin_{\bm{\xi} \succeq \textbf{0}}\ \mathcal{L}({\bf w}^{(k+1)}, b^{(k+1)}, {\bf z}^{(k+1)}, \bm{\xi}, {\bf s}^{(k)}, {\bf u}^{(k)}, {\bf v}^{(k)}),\label{eq8}\\
&{\bf s}^{(k+1)}=\argmin_{{\bf s} \succeq \textbf{0}}\ \mathcal{L}({\bf w}^{(k+1)}, b^{(k+1)}, {\bf z}^{(k+1)}, \bm{\xi}^{(k+1)}, {\bf s}, {\bf u}^{(k)}, {\bf v}^{(k)}),\label{eq9}\\
&{\bf u}^{(k+1)}={\bf u}^{(k)}+({\bf w}^{(k+1)}-{\bf z}^{(k+1)}),\label{eq10}\\
&{\bf v} ^{(k+1)}={\bf v}^{(k)}+(\bm{\xi}^{(k+1)}-{\bf s}^{(k+1)}+{\bf Hw}^{(k+1)}+b{\bf y}-{\bf 1}).\nonumber\\\label{eq11}
\end{align}


Considering optimizing problem~\eqref{eq5}, we can obtain the closed-form solution of it via $\partial{\mathcal{L}}/\partial{{\bf w}}={\bf 0}$, that is,
\begin{equation}
\label{wsolution}
\begin{split}
{\bf w}^{(k+1)}=(\rho_1{\bf I}_{d}+
\rho_2{\bf H}^\top {\bf H})^{-1}[\rho_1({\bf z}^{(k)}-{\bf u}^{(k)})+ \\   \rho_2{\bf H}^\top({\bf s}^{(k)}+{\bf 1}-\bm{\xi}^{(k)}-{\bf v}^{(k)}-b^{(k)}{\bf y})],
\end{split}
\end{equation}
where ${\bf I}_{d}$ denotes the $d\times d$ identity matrix.

Note that~\eqref{wsolution} requires to calculate the inversion of a ${d \times d}$ matrix. The computational cost is especially high for the large $d$ case.
Therefore, we further investigate an efficient solution for the update of ${\bf w}$ according to the value of $n$ and $d$.

Let $\rho=\rho_1/\rho_2$ and ${\bf f}^{(k)}=\rho({\bf z}^{(k)}-{\bf u}^{(k)})+{\bf H}^\top({\bf s}^{(k)}+{\bf 1}-\bm{\xi}^{(k)}-{\bf v}^{(k)}-b^{(k)}{\bf y})$, then Equation~\eqref{wsolution} can be equivalently converted to
\begin{equation}
\label{wsolution1}
\begin{split}
{\bf w}^{(k+1)}=(\rho{\bf I}_d+ {\bf H}^\top {\bf H})^{-1}{\bf f}^{(k)}.
\end{split}
\end{equation}

If $d$ is more than $n$ in order, we can apply the Sherman-Morrison formula \cite{sherman1950adjustment} to solve (\ref{wsolution1}). Therefore, we have
\begin{equation}
\label{wsolution12}
\begin{split}
{\bf w}^{(k+1)}=\frac{{\bf f}^{(k)}}{\rho}-\frac{({\bf H}^\top({\bf U}^{-1}({\bf L}^{-1}({\bf H}{\bf f}^{(k)}))))}{\rho^{2}},
\end{split}
\end{equation}
where ${\bf L}$ and ${\bf U}$ are the Cholesky factorization of a $n\times n$ positive definite matrix ${\bf I}_n+ \frac{1}{\rho}{\bf H}{\bf H}^\top$, i.e., ${\bf I}_n+ \frac{1}{\rho}{\bf H} {\bf H}^\top={\bf LU}$. Here, ${\bf I}_n$ is the $n\times n$ identity matrix.

For the case when $n\geq d$, observe that the matrix $\rho{\bf I}_d+ {\bf H}^\top {\bf H}$ is positive definite, then we can obtain the solution of ${\bf w}^{(k+1)}$ via
\begin{equation}
\label{wsolution123}
\begin{split}
{\bf w}^{(k+1)}={\bf U}^{-1}({\bf L}^{-1}{\bf f}^{(k)}),
\end{split}
\end{equation}
 where ${\bf L}$ and ${\bf U}$ are the Cholesky factorization of a $d\times d$ matrix $\rho{\bf I}_d+ {\bf H}^\top {\bf H}$, i.e., $\rho{\bf I}_d+ {\bf H}^\top {\bf H}={\bf LU}$.



Consequently, equation (\ref{wsolution}) can be equivalently converted to
\begin{equation}
\label{wsolution2}
{\bf w}^{(k+1)}=\left\{\begin{array}{ll}
{\bf U}^{-1}({\bf L}^{-1}{\bf f}^{(k)}), & \text{if} \ n \geq d, \\
\frac{{\bf f}^{(k)}}{\rho}-\frac{({\bf H}^\top({\bf U}^{-1}({\bf L}^{-1}({\bf H}{\bf f}^{(k)}))))}{\rho^{2}} & \textnormal{otherwise,} \\
\end{array}\right.
\end{equation}
where ${\bf L}$ and ${\bf U}$ are the Cholesky factorization of $\rho{\bf I}_d+ {\bf H}^\top {\bf H}$, if $n\geq d$, and the Cholesky factorization of ${\bf I}_n+ \frac{1}{\rho}{\bf H} {\bf H}^\top$ otherwise.

\newtheorem{myTheo}{Proposition}
\begin{myTheo}
	\label{theo1}
	For the case when $d\gg n$, the computational cost of the reformulated {\bf w}-update by Equation (\ref{wsolution2}) is $ O(dn^2)$ flops, giving rise to an improvement by a factor of $(d/n)^2$ over the naive {\bf w}-update by Equation (\ref{wsolution}).
	
\end{myTheo}
\label{proof1}
\begin{proof}
	
	This proof exploits no structure in ${\bf H}$, i.e., our generic method works for any matrix. For convenience, this proof neglects the superscripts of each variable.
	
	For the reformulated ${\bf w}$-update by Equation  (\ref{wsolution2}), we can first form ${\bf f}=\rho({\bf z}-{\bf u})+{\bf H}^\top({\bf s}+{\bf 1}-{\bm \xi}-{\bf v}-b{\bf y})$ at a cost of $O(dn)$ flops. Then forming ${\bf C}={\bf I}_n+\frac{1}{\rho}{\bf H}{\bf H}^\top$ costs $O(dn^2)$ flops, followed by the calculation of ${\bf C}={\bf L}{\bf U}$ via Cholesky factorization at a cost of $O(n^3)$ flops. After that, we can form ${\bf t}_1={\bf H}^\top({\bf U}^{-1}({\bf L}^{-1}({\bf Hf})))$ through two matrix-vector multiplications and two back-solve steps at a cost of $O(dn)$ flops.
	Since it costs $O(d)$ flops for forming $\frac{{\bf f}}{\rho}-\frac{{\bf t}_1}{\rho^2}$  and $O(dn+dn^2+n^3+dn+d)=O(dn^2)$, the overall cost of forming $\frac{\bf f}{\rho}-\frac{({\bf H}^\top({\bf U}^{-1}({\bf L}^{-1}({\bf Hf}))))}{\rho^2}$ is $O(dn^2)$ flops.
	
	In terms of the naive update by Equation (\ref{wsolution}), we can first form ${\bf t}_2=\rho_1({\bf z}-{\bf u})+\rho_2{\bf H}^\top({\bf s}+{\bf 1}-{\bm \xi}-{\bf v}-b{\bf y})$ at a cost of  $O(dn)$ flops. Because $d$ is more than $n$ in order, we can form ${\bf T}=(\rho_1{\bf I}_d+\rho_2{\bf H}^\top {\bf H})^{-1}$ at a cost of $O(d^2n+d^3)=O(d^3)$ flops. Considering that ${\bf t}_2\in\mathbb{R}^{d}$ and ${\bf T}\in\mathbb{R}^{d\times d}$, the cost of forming ${\bf T}{\bf t}_2$ is $O(d^2)$ flops. Thus, the naive method for calculating $(\rho_1{\bf I}_d+\rho_2{\bf H}^\top {\bf H})^{-1}[\rho_1({\bf z}-{\bf u})+\rho_2{\bf H}({\bf s}+{\bf 1}-{\bm \xi}-{\bf v}-b{\bf y})]$ costs $O(dn+d^2+d^3)=O(d^3)$ flops in total.

	Since $d^3/dn^2=(d/n)^2$, thus the reformulated method obtains an improvement by a factor of $O(d/n)^2$ over the naive method. This completes the proof of Proposition \ref{theo1}.
\end{proof}

By letting $\partial{\mathcal{L}}/\partial{b}=0$, we obtain the solution of Equation (\ref{eq6}), that is,
\begin{equation}
\label{eq13}
b^{(k+1)}=\frac{{\bf y}^\top({\bf s}^{(k)}+{\bf 1}-{\bf Hw}^{(k+1)}-\bm{\xi}^{(k)}-{\bf v}^{(k)})}{{\bf y}^\top {\bf y}}.
\end{equation}

In addition, note that Equation~\eqref{eq7} is equivalent to optimizing the following problem:
\begin{equation}
\label{eq14}
{\bf z}^{(k+1)}=\argmin_{\bf z} \ \frac{1}{2}||{\bf z}-({\bf w}^{(k+1)}+{\bf u}^{(k)})||_2^2+\frac{1}{\rho_1} P({\bf z}).
\end{equation}
Based on the observation that $P({\bf z})=\sum_{i=1}^{d}p_\lambda(z_i)$, we can get the solution of problem \eqref{eq14} via solving $d$ independent univariate optimization problems. Let $\bm{\psi}^{(k+1)} = {\bf w}^{(k+1)}+{\bf u}^{(k)}$, then we can obtain the solution of the $i$th entry of variable ${\bf z}$ in the $(k+1)$th iteration, that is,
\begin{equation}
\label{scadbefore}
z_i^{(k+1)}=\argmin_{z_i} \  \frac{1}{2}(z_i-\psi_i^{(k+1)})^2+\frac{1}{\rho_1}p_\lambda(z_i), \  \forall i\in[1,d].
\end{equation}
It has been shown in \cite{gong2013a} that this subproblem admits a closed-form solution for many commonly used nonconvex penalties.
The closed-form solution of ${z}_i^{(k+1)}$ for four commonly used nonconvex regularizers including LSP, SCAD penalty, MCP and capped-$\ell_1$ penalty are shown in the Appendix \ref{zupdate}.

The closed-form solution of Equation~\eqref{eq8} can be obtained by performing $\partial{\mathcal{L}}/\partial{{\bm \xi}}={\bf 0}$, followed by the projection to the 1-dimensional nonnegative set($[0,+\infty)$), that is,
\begin{align}
&\bm{\xi}^{(k+\frac{1}{2})}={\bf s}^{(k)}+{\bf 1}-{\bf v}^{(k)}-{\bf H}{\bf w}^{(k+1)}-b^{(k+1)}{\bf y}-\frac{{\bf 1}}{n\rho_2},\label{eq1410}\\
&{\bm \xi}^{(k+1)}=\max({\bm \xi}^{(k+\frac{1}{2})},{\bf 0}).\label{eq141}
\end{align}

Similarly, the solution of~\eqref{eq9} can be calculated through $\partial{\mathcal{L}}/\partial{{\bf s}}={\bf 0}$. Therefore, we can perform a two-step update as follows.
\begin{align}
&{\bf s}^{(k+\frac{1}{2})}={\bf H}{\bf w}^{(k+1)}+b^{(k+1)}{\bf y}+\bm{\xi}^{(k+1)}-{\bf 1}+{\bf v}^{(k)},\label{eq1420} \\
&{\bf s}^{(k+1)}=\max({\bf s}^{(k+\frac{1}{2})},{\bf 0}).\label{eq1421}
\end{align}

\renewcommand{\algorithmicrequire}{\textbf{Input:}}
\renewcommand{\algorithmicensure}{\textbf{Output:}}
\begin{algorithm}[t]
	\caption{ADMM for Nonconvex Penalized SVMs}
	\label{alg1}
	\begin{algorithmic}[1]
		\Require
		training data $\mathcal{S}$, parameter $\rho_1>0$, $\rho_1>0$, $\lambda$, $\theta$
		\State Initialize ${\bf w}^{(0)}$, $b^{(0)}$, ${\bf z}^{(0)}$, $\bm{\xi}^{(0)}$, ${\bf s}^{(0)}$, ${\bf u}^{(0)}$, ${\bf v}^{(0)}, k\leftarrow 0$
		\State Calculate ${\bf H}={\bf Y}{\bf X}$,  and $\rho=\rho_1/\rho_2$
		\State \textbf{if} $n\geq d$
		\State \ \ \ \ \ Form ${\bf C}=\rho{\bf I}_d+ {\bf H}^\top {\bf H}$.
		\State \textbf{else}
		\State \ \ \ \ \ Form ${\bf C}={\bf I}_n+ \frac{1}{\rho}{\bf H}{\bf H}^\top$.
		\State \textbf{end if}
		\State Calculate Cholesky factorization of ${\bf C}$ (${\bf C}$=${\bf LU}$).
		\State \textbf{repeat}
		\State \ \ \ \ Calculate ${\bf f}^{(k)}=\rho({\bf z}^{(k)}-{\bf u}^{(k)})+{\bf H}^\top({\bf s}^{(k)}+{\bf 1}-\bm{\xi}^{(k)}-{\bf v}^{(k)}-b^{(k)}{\bf y})$
		\State \ \ \ Calculate ${\bf w}^{(k+1)}$:
		\begin{displaymath} \quad \  {\bf w}^{(k+1)}=\left\{\begin{array}{ll}
		{\bf U}^{-1}({\bf L}^{-1}{\bf f}^{(k)}), & \text{if} \ n \geq d, \\
		\frac{{\bf f}^{(k)}}{\rho}-\frac{({\bf H}^\top({\bf U}^{-1}({\bf L}^{-1}({\bf H}{\bf f}^{(k)}))))}{\rho^{2}} & \textnormal{otherwise.} \\
		\end{array}\right.\end{displaymath}
		
		\State \ \ \ Calculate $b^{(k+1)}=\frac{{\bf y}^\top({\bf s}^{(k)}+{\bf 1}-{\bf Hw}^{(k+1)}-\bm{\xi}^{(k)}-{\bf v}^{(k)})}{{\bf y}^\top {\bf y}}$.
		\State  \ \ \ Calculate ${\bf z}^{(k+1)}$:
		\begin{displaymath} \ \ \ \ \ {\bf z}^{(k+1)}=\argmin_{{\bf z}} \ \frac{1}{2}||{\bf z}-({\bf w}^{(k+1)}+{\bf u}^{(k)})||_2^2+\frac{1}{\rho_1} P({\bf z}) \end{displaymath}.
		\State \ \ \ Calculate $\bm{\xi}^{(k+\frac{1}{2})}={\bf s}^{(k)}+{\bf 1}-{\bf v}^{(k)}-{\bf H}{\bf w}^{(k+1)}-b^{(k+1)}{\bf y}-\frac{{\bf 1}}{n\rho_2}$.
		\State \ \ \ \ Calculate $\bm{\xi}^{(k+1)}=\max(\bm{\xi}^{(k+\frac{1}{2})},{\bf 0})$.
		\State \ \ \ Calculate ${\bf s}^{(k+\frac{1}{2})}={\bf H}{\bf w}^{(k+1)}+b^{(k+1)}{\bf y}+\bm{\xi}^{(k+1)}-{\bf 1}+{\bf v}^{(k)}$;
		\State \ \ \  Calculate ${\bf s}^{(k+1)}=\max({\bf s}^{(k+\frac{1}{2})},{\bf 0})$.
		\State \ \ \ Calculate ${\bf u}^{(k+1)}={\bf u}^{(k)}+({\bf w}^{(k+1)}-{\bf z}^{(k+1)})$.
		\State \ \ \ \ Calculate ${\bf v}^{(k+1)}={\bf v}^{(k)}+(\bm{\xi}^{(k+1)}-{\bf s}^{(k+1)}+{\bf Hw}^{(k+1)}+b{\bf y}-{\bf 1})$.
		\State \ \ \  $k\leftarrow k+1$.
		\State \textbf{until} stopping criterion is satisfied.
		\Ensure the solution ${\bf w}^\star$ and $b^\star$
	\end{algorithmic}
\end{algorithm}

\subsection{Algorithm and Computational Cost Analysis} \label{algorimanalysis}
The procedure for solving nonconvex penalized SVMs via ADMM is shown in Algorithm \ref{alg1}.
It mainly consists of two parts: the pre-computation stage (line 1-8) and the iteration stage (line 9-21).

In Algorithm \ref{alg1}, the primal and dual variables are initialized first at line 1, followed by the calculation of two constant variable ${\bf H}$ and ${\rho}$ at line 2.
Since ${\bf Y}$ is a diagonal matrix, line 2 can be carried out at a total cost of $O(n^2)$ flops.
Note that the parameter $\rho_1$ and $\rho_2$ remain unchanged throughout the ADMM procedure. Thus we can carry out the Cholesky factorization according to the value of $d$ and $n$ once, and then use this cached factorization in subsequent solve steps. In algorithm \ref{alg1}, we first form an intermediate variable ${\bf C}$, a $d\times d$ or $n\times n$ matrix, according to the value of $d$ and $n$, and then factor it (line 3-8).
According to analysis arising in the proof of proposition \ref{theo1}, forming ${\bf C}$ and then factoring it cost $O(dn^2)$
flops when the order of $d$ is more than $n$. Meanwhile, if $d$ is on the order of or less than $n$, line 3-8 can be carried out at a cost of $O(d^2n)$ flops.
Therefore, we can see that the overall cost of carrying out the pre-computation stage is $O(dn^2)$ flops, if $d>n$, and $O(d^2n)$ flops otherwise.

After the pre-computation stage, Algorithm \ref{alg1} begins to iterate the ADMM procedure and quits until the pre-defined stopping criterion is satisfied (line 9-21). For the ${\bf w}$-update, ${\bf f}^{(k)}$ can be first obtained via performing line 10 at a cost of $O(dn)$ flops.
Then if the order of $d$ is more than $n$, we can see that $\frac{{\bf f}^{(k)}}{\rho}-\frac{({\bf H}^\top({\bf U}^{-1}({\bf L}^{-1}({\bf H}{\bf f}^{(k)}))))}{\rho^{2}}$ can be formed at a cost of $O(dn)$ flops according to the analysis arising in the proof of proposition \ref{theo1}. Otherwise, it takes $O(d^2)$ flops to form ${\bf U}^{-1}({\bf L}^{-1}{\bf f}^{(k)})$ via two back-solve steps.
Since  $O(dn)+O(dn)=O(dn)$ and $O(dn)+O(d^2)=O(dn)$,
the ${\bf w}$-update costs $O(dn)$ flops in any case. In terms of the update of ${\bf z}$, it has been shown that we can get the solution of ${\bf z}^{(k+1)}$ by solving $d$ independent univariate optimization problems and each of these univariate optimization problems owns a closed-form solution. Therefore, this step can be carried out at a cost of $O(d)\times O(1)=O(d)$ flops. Moreover, line 12 and line 14-20 can be easily carried out at a cost of $O(dn)$ flops in total.
Since $O(dn)+O(d)+O(dn)=O(dn)$, thus it takes $O(dn)$ flops per iteration.

As a result, we can see that the overall computational cost of Algorithm \ref{alg1} is
\begin{equation}
\label{complexity}
\left\{\begin{array}{ll}
{O(d^2n)+O(dn)\times\#iterations} \ \ \text{if} \ n \geq d, \\
{O(dn^2)+O(dn)\times\#iterations} \ \ \textnormal{otherwise.} \\
\end{array}\right.
\end{equation}

The computational complexity shown in (\ref{complexity}) demonstrates the efficiency of the proposed algorithm.
In addition, note that the computational complexity analysis discussed above does not consider the sparse structure of the feature matrix. When exploring the sparsity, the overall computational complexity of Algorithm \ref{alg1} can be further decreased.
Meanwhile, it has been shown that ADMM can converge to modest accuracy-sufficient for many applications-within a few tens of iterations in \cite{boyd2011distributed}. The experimental results also demonstrate this point. We find that Algorithm \ref{alg1} always converges within only a few tens of iterations to get a reasonable result by appropriately tuning the parameter $\lambda$, $\theta$, $\rho_1$ and $\rho_2$.

\section{Convergence Analysis}\label{convergenceanalysis}
In this section, we present the detailed convergence analysis of the proposed algorithm. To present the analysis, we first modify a little about the scheme for updating ${\bf z}^{(k+1)}$, that is,
\begin{align}\label{eq7'}
{\bf z}^{(k+1)}&=\argmin_{\bf z}\ \mathcal{L}({\bf w}^{(k+1)}, b^{(k+1)}, {\bf z}, \bm{\xi}^{(k)}, {\bf s}^{(k)}, {\bf u}^{(k)}, {\bf v}^{(k)})\nonumber\\
&+\frac{\beta}{2}\|{\bf z}-{\bf z}^{(k)}\|^2,
\end{align}
where $\beta>0$ but is small. If $\beta=0$, \eqref{eq7'} equals to \eqref{eq7}; and if $\beta>0$  is very small,  \eqref{eq7'} is very close to \eqref{eq7}.
After that, we give the convergence analysis following the proof framework built in \cite{wang2015global}, which is also used in \cite{sun2017alternating,sun2017iteratively}. However, it's noting that our work is totally not an simple extension of \cite{wang2015global}. As mentioned before, \cite{wang2015global} cannot be applied to solve the nonconvex penalized hinge loss function since it requires the loss function to be differentiable.

Before giving the convergence analysis, We need following two assumptions.

\begin{assumption}\label{ass1}
For any $k$, ${\bf v}^{(k)}\in \emph{Im}({\bf y})$.
\end{assumption}

\begin{assumption}\label{ass2}
The augmented Lagrangian function
$\mathcal{L}({\bf w},b,{\bf z},\bm{\xi},{\bf s},{\bf u},{\bf v})$ has a lower bound, that is, $\inf \mathcal{L}({\bf w},b,{\bf z},\bm{\xi},{\bf s},{\bf u},{\bf v})>-\infty$.
\end{assumption}

Now we introduce several  definitions and properties needed in the analysis.
\begin{definition}\label{sconvex}
We say $f(x)$ is strongly convex with constant $\delta\geq 0$, if the function $f(x)-\frac{\delta\|x\|^2}{2}$ is also convex.
\end{definition}

If a function is strongly convex, the following fact obviously holds:
Let $x^*$ be a minimizer of $f$ which is strongly convex with constant $\delta$. Then, it holds that
\begin{align}\label{strongconvex}
f(x)-f(x^*)\geq\frac{\delta}{2}\|x-x^*\|^2.
\end{align}

To simplify the presentation, we use
$${\bf D}^{(k)}:=({\bf w}^{(k)}, b^{(k)}, {\bf z}^{(k)},  \bm{\xi}^{(k)}, {\bf s}^{(k)}, {\bf u}^{(k)}, {\bf v}^{(k)}).$$

Now, we are prepared to present the convergence analysis of our algorithm.
\begin{lemma}
Let $\{{\bf D}^{(k)}\}_{k=0,1,2,\ldots}$ be generated by our algorithm, then,
\begin{align}
&\mathcal{L}({\bf w}^{(k)}, b^{(k)}, {\bf z}^{(k)},  \bm{\xi}^{(k)}, {\bf s}^{(k)}, {\bf u}^{(k)}, {\bf v}^{(k)})\nonumber\\
&\geq\mathcal{L}({\bf w}^{(k+1)}, b^{(k+1)}, {\bf z}^{(k+1)},  \bm{\xi}^{(k+1)}, {\bf s}^{(k+1)}, {\bf u}^{(k)}, {\bf v}^{(k)})\nonumber\\
&+\frac{\nu}{2}\|{\bf D}^{(k+1)}-{\bf D}^{(k)}\|^2,
\end{align}
where
$$\nu:=\min\{\rho_1+\rho_2\sigma_{\min}({\bf H}^{\top}{\bf H}),\rho_2,\rho_2\|{\bf y}\|^2,\beta\}.$$
\end{lemma}
\begin{proof}
Noting $\mathcal{L}({\bf w}, b^{(k)}, {\bf z}^{(k)},  \bm{\xi}^{(k)}, {\bf s}^{(k)}, {\bf u}^{(k)}, {\bf v}^{(k)})$ is strongly convex with $\rho_1+\rho_2\sigma_{\min}({\bf H}^{\top}{\bf H})$ with respect to ${\bf w}$. Thus, minimization of $\mathcal{L}({\bf w}, b^{(k)}, {\bf z}^{(k)},  \bm{\xi}^{(k)}, {\bf s}^{(k)}, {\bf u}^{(k)}, {\bf v}^{(k)})$  directly yields
\begin{align}\label{4.1}
&\mathcal{L}({\bf w}^{(k)}, b^{(k)}, {\bf z}^{(k)},  \bm{\xi}^{(k)}, {\bf s}^{(k)}, {\bf u}^{(k)}, {\bf v}^{(k)})\nonumber\\
&\geq\mathcal{L}({\bf w}^{(k+1)}, b^{(k)}, {\bf z}^{(k)},  \bm{\xi}^{(k)}, {\bf s}^{(k)}, {\bf u}^{(k)}, {\bf v}^{(k)})\nonumber\\
&+\frac{\rho_1+\rho_2\sigma_{\min}({\bf H}^{\top}{\bf H})}{2}\|{\bf w}^{(k+1)}-{\bf w}^{(k)}\|^2.
\end{align}
Similarly, we can obtain the following inequalities
\begin{align}\label{4.2}
&\mathcal{L}({\bf w}^{(k+1)}, b^{(k)}, {\bf z}^{(k)},  \bm{\xi}^{(k)}, {\bf s}^{(k)}, {\bf u}^{(k)}, {\bf v}^{(k)})\nonumber\\
&\geq\mathcal{L}({\bf w}^{(k+1)}, b^{(k+1)}, {\bf z}^{(k)},  \bm{\xi}^{(k)}, {\bf s}^{(k)}, {\bf u}^{(k)}, {\bf v}^{(k)})\nonumber\\
&+\frac{\rho_2\|{\bf y}\|^2}{2}\|b^{(k+1)}-b^{(k)}\|^2,
\end{align}
and
\begin{align}\label{4.3}
&\mathcal{L}({\bf w}^{(k+1)}, b^{(k+1)}, {\bf z}^{(k+1)},  \bm{\xi}^{(k)}, {\bf s}^{(k)}, {\bf u}^{(k)}, {\bf v}^{(k)})\nonumber\\
&\geq\mathcal{L}({\bf w}^{(k+1)}, b^{(k+1)}, {\bf z}^{(k+1)},  \bm{\xi}^{(k+1)}, {\bf s}^{(k)}, {\bf u}^{(k)}, {\bf v}^{(k)})\nonumber\\
&+\frac{\rho_2}{2}\|{\bf \xi}^{(k+1)}-{\bf \xi}^{(k)}\|^2,
\end{align}
and
\begin{align}\label{4.4}
&\mathcal{L}({\bf w}^{(k+1)}, b^{(k+1)}, {\bf z}^{(k+1)},  \bm{\xi}^{(k+1)}, {\bf s}^{(k)}, {\bf u}^{(k)}, {\bf v}^{(k)})\nonumber\\
&\geq\mathcal{L}({\bf w}^{(k+1)}, b^{(k+1)}, {\bf z}^{(k+1)},  \bm{\xi}^{(k+1)}, {\bf s}^{(k+1)}, {\bf u}^{(k)}, {\bf v}^{(k)})\nonumber\\
&+\frac{\rho_2}{2}\|{\bf s}^{(k+1)}-{\bf s}^{(k)}\|^2.
\end{align}
Noting  ${\bf z^{(k+1)}}$ is the minimizer  of $\mathcal{L}({\bf w}^{(k+1)}, b^{(k+1)}, {\bf z},  \bm{\xi}^{(k+1)}, {\bf s}^{(k)}, {\bf u}^{(k)}, {\bf v}^{(k)})+\frac{\beta\|{\bf z}-{\bf z}^{(k)}\|^2}{2}$ with respect to ${\bf z}$,
which means
\begin{align}\label{4.5}
&\mathcal{L}({\bf w}^{(k+1)}, b^{(k+1)}, {\bf z}^{(k)},  \bm{\xi}^{(k+1)}, {\bf s}^{(k)}, {\bf u}^{(k)}, {\bf v}^{(k)})\nonumber\\
&\geq\mathcal{L}({\bf w}^{(k+1)}, b^{(k+1)}, {\bf z}^{(k+1)},  \bm{\xi}^{(k+1)}, {\bf s}^{(k+1)}, {\bf u}^{(k)}, {\bf v}^{(k)})\nonumber\\
&+\frac{\beta\|{\bf z}^{(k+1)}-{\bf z}^{(k)}\|^2}{2}.
\end{align}
Summing \eqref{4.1}-\eqref{4.5} yields
\begin{align}
&\mathcal{L}({\bf w}^{(k)}, b^{(k)}, {\bf z}^{(k)},  \bm{\xi}^{(k)}, {\bf s}^{(k)}, {\bf u}^{(k)}, {\bf v}^{(k)})\nonumber\\
&\geq\mathcal{L}({\bf w}^{(k+1)}, b^{(k+1)}, {\bf z}^{(k+1)},  \bm{\xi}^{(k+1)}, {\bf s}^{(k+1)}, {\bf u}^{(k)}, {\bf v}^{(k)})\nonumber\\
&+\frac{\nu}{2}\|{\bf D}^{(k+1)}-{\bf D}^{(k)}\|^2,
\end{align}
where
$$\nu:=\min\{\rho_1+\rho_2\sigma_{\min}({\bf H}^{\top}{\bf H}),\rho_2,\rho_2\|{\bf y}\|^2,\beta\}.$$
\end{proof}

\begin{lemma}
If Assumption \ref{ass1} holds,
\begin{align}\label{ca1}
&\|{\bf v}^{(k+1)}-{\bf v}^{(k)}\|^2\leq c_1\|{\bf D}^{(k+1)}-{\bf D}^{(k)}\|^2\nonumber\\
&\quad\quad+c_2\|{\bf D}^{(k+2)}-{\bf D}^{(k+1)}\|^2,
\end{align}
and
\begin{align}\label{ca2}
&\|{\bf u}^{(k+1)}-{\bf u}^{(k)}\|^2\leq
c_3\|{\bf D}^{(k+1)}-{\bf D}^{(k)}\|^2\nonumber\\
&\quad\quad+c_4\|{\bf D}^{(k+2)}-{\bf D}^{(k+1)}\|^2,
\end{align}
where $c_1,c_2>0$ is a polynomial composition of $\|{\bf H}\|,\|{\bf y}\|$, and $c_3=O(\frac{1}{\rho_1^2})$, and $c_4=O(\frac{1}{\rho_1^2})$.
\end{lemma}
\begin{proof}
The optimization condition for updating ${\bf w}^{(k+1)}$ gives
\begin{align}
\nabla\mid_{{\bf w}={\bf w}^{(k+1)}}\mathcal{L}({\bf w}, b^{(k)}, {\bf z}^{(k)},  \bm{\xi}^{(k)}, {\bf s}^{(k)}, {\bf u}^{(k)}, {\bf v}^{(k)})=\textbf{0}.
\end{align}
That is also
\begin{align}\label{3w}
&\rho_1({\bf w}^{(k+1)}-{\bf z}^{(k)}+{\bf u}^{(k)})+\rho_1 {\bf H}^{\top}({\bf H}{\bf w}^{(k+1)}\nonumber\\
&+b^{(k)}{\bf y}+\bm{\xi}^{(k)}-{\bf s}^{(k)}-{\bf 1}+{\bf v}^{(k)})=\textbf{0}.
\end{align}
On the other hand, the optimization condition for updating $b^{(k+1)}$ gives
\begin{align}
\nabla\mid_{b=b^{(k+1)}}\mathcal{L}({\bf w}^{(k+1)}, b, {\bf z}^{(k)},  {\bf \xi}^{(k)}, {\bf s}^{(k)}, {\bf u}^{(k)}, {\bf v}^{(k)})=\textbf{0}.
\end{align}
That can be represented as
\begin{align}\label{3b}
{\bf y}^{\top} ({\bf H}{\bf w}^{(k+1)}+b^{(k+1)}{\bf y}+\bm{\xi}^{(k)}-{\bf s}^{(k)}-{\bf 1}+{\bf v}^{(k)})=\textbf{0}.
\end{align}
In \eqref{3b}, letting $k=k+1$,
\begin{align}\label{3b-2}
{\bf y}^{\top} ({\bf H}{\bf w}^{(k+2)}+b^{(k+2)}{\bf y}+\bm{\xi}^{(k+1)}-{\bf s}^{(k+1)}-{\bf 1}+{\bf v}^{(k+1)})=\textbf{0}.
\end{align}
Subtraction of \eqref{3b} and \eqref{3b-2} gives
\begin{align}
&\|{\bf y}^{\top}({\bf v}^{(k+1)}-{\bf v}^{(k)})\|\leq  \|{\bf H}\|\|{\bf y}\|\|{\bf w}^{(k+2)}-{\bf w}^{(k+1)}\|\nonumber\\
&+\|{\bf y}\|\|b^{(k+2)}-b^{(k+1)}\|+\|{\bf y}\|\| \bm{\xi}^{(k+1)}-\bm{\xi}^{(k)}\|\nonumber\\
&+\|{\bf y}\|\| {\bf s}^{(k+1)}-{\bf s}^{(k)}\|.
\end{align}
With Assumption \ref{ass1}, we get
\begin{align}\label{ca1-l}
&\|{\bf v}^{(k+1)}-{\bf v}^{(k)}\|^2\leq c_1\|{\bf D}^{(k+1)}-{\bf D}^{(k)}\|^2\nonumber\\
&\quad\quad+c_2\|{\bf D}^{(k+2)}-{\bf D}^{(k+1)}\|^2,
\end{align}
where $c_1,c_2>0$ is a polynomial composition of $\|{\bf H}\|,\|{\bf y}\|$.
In \eqref{3w}, letting $k=k+1$,
\begin{align}\label{3w-2}
&\rho_1({\bf w}^{(k+2)}-{\bf z}^{(k+1)}+{\bf u}^{(k+1)})+\rho_1 {\bf H}^{\top}({\bf H}{\bf w}^{(k+2)}\nonumber\\
&+b^{(k+1)}{\bf y}+\bm{\xi}^{(k+1)}-{\bf s}^{(k+1)}-{\bf 1}+{\bf v}^{(k+1)})=\textbf{0}.
\end{align}
Similarly, Subtraction of \eqref{3w} and \eqref{3w-2} tells us
\begin{align}\label{ca2'}
&\|{\bf u}^{(k+1)}-{\bf u}^{(k)}\|^2\leq \hat{c}_3\|{\bf D}^{(k+1)}-{\bf D}^{(k)}\|^2\nonumber\\
&\quad\quad+\hat{c}_4\|{\bf v}^{(k)}-{\bf v}^{(k)}\|^2,
\end{align}
where $\hat{c}_3,\hat{c}_4>0$ is a polynomial composition of $\rho_1,\|{\bf H}\|,\|{\bf y}\|$ and $\hat{c}_3=O(\frac{1}{\rho_1^2})$, and $\hat{c}_4=O(\frac{1}{\rho_1^2})$.
Using \eqref{ca2'} to \eqref{ca1}, we then get
\begin{align}\label{ca2-l}
&\|{\bf u}^{(k+1)}-{\bf u}^{(k)}\|^2\leq
c_3\|{\bf D}^{(k+1)}-{\bf D}^{(k)}\|^2\nonumber\\
&\quad\quad+c_4\|{\bf D}^{(k+2)}-{\bf D}^{(k+1)}\|^2,
\end{align}
where $c_3=O(\frac{1}{\rho_1^2})$, and $c_4=O(\frac{1}{\rho_1^2})$.
\end{proof}

\begin{theorem}
If Assumptions \ref{ass1} and \ref{ass2} hold, and
\begin{align}\label{para-re}
\frac{\nu}{2}>\rho_1c_3+\rho_2c_1+\rho_2c_2+\rho_1c_4.
\end{align}
Then,
$$\lim_{k}\|{\bf D}^{(k+1)}-{\bf D}^{(k)}\|=0.$$
\end{theorem}

\begin{proof}
With direct calculations, we can derive
\begin{align}\label{des}
&\mathcal{L}({\bf w}^{(k)}, b^{(k)}, {\bf z}^{(k)},  \bm{\xi}^{(k)}, {\bf s}^{(k)}, {\bf u}^{(k)}, {\bf v}^{(k)})\nonumber\\
&\geq\mathcal{L}({\bf w}^{(k+1)}, b^{(k+1)}, {\bf z}^{(k+1)},  \bm{\xi}^{(k+1)}, {\bf s}^{(k+1)}, {\bf u}^{(k+1)}, {\bf v}^{(k+1)}\nonumber\\
&+\frac{\nu}{2}\|{\bf D}^{(k+1)}-{\bf D}^{(k)}\|^2-\rho_1\|{\bf u}^{(k+1)}-{\bf u}^{(k)}\|^2\nonumber\\
&-\rho_2\|{\bf v}^{(k+1)}-{\bf v}^{(k)}\|^2
\end{align}

Substituting \eqref{ca1} and \eqref{ca2} to \eqref{des},
\begin{align}\label{des2}
&\mathcal{L}({\bf w}^{(k)}, b^{(k)}, {\bf z}^{(k)},  \bm{\xi}^{(k)}, {\bf s}^{(k)}, {\bf u}^{(k)}, {\bf v}^{(k)})\nonumber\\
&\geq\mathcal{L}({\bf w}^{(k+1)}, b^{(k+1)}, {\bf z}^{(k+1)},  \bm{\xi}^{(k+1)}, {\bf s}^{(k+1)}, {\bf u}^{(k+1)}, {\bf v}^{(k+1)}\nonumber\\
&+(\frac{\nu}{2}-\rho_1c_3-\rho_2c_1)\|{\bf D}^{(k+1)}-{\bf D}^{(k)}\|^2\nonumber\\
&-(\rho_2c_2+\rho_1c_4)\|{\bf D}^{(k+2)}-{\bf D}^{(k+1)}\|^2.
\end{align}
Denote $a_k:=\mathcal{L}({\bf w}^{(k)}, b^{(k)}, {\bf z}^{(k)},  \bm{\xi}^{(k)}, {\bf s}^{(k)}, {\bf u}^{(k)}, {\bf v}^{(k)})+(\frac{\nu}{2}-\rho_1c_3-\rho_2c_1)\|{\bf D}^{(k+1)}-{\bf D}^{(k)}\|^2.$
Then we can see \eqref{des2} actually indicates
\begin{align}\label{des3}
&\left[\frac{\nu}{2}-(\rho_1c_3+\rho_2c_1+\rho_2c_2+\rho_1c_4)\right]\|{\bf D}^{(k+2)}-{\bf D}^{(k+1)}\|^2\nonumber\\
&\leq a_k-a_{k+1}.
\end{align}
With Assumption \ref{ass2}, $\inf_{k}\{a_k\}>-\infty$, and then
$$\sum_{k}(a_k-a_{k+1})<+\infty.$$
Thus, we get
\begin{align}\label{des4}
&\sum_{k}\left[\frac{\nu}{2}-(\rho_1c_3+\rho_2c_1+\rho_2c_2+\rho_1c_4)\right]\|{\bf D}^{(k+2)}-{\bf D}^{(k+1)}\|^2\nonumber\\
&<+\infty.
\end{align}
That means
\begin{align}\label{vanish}
\lim_{k}\|{\bf D}^{(k+1)}-{\bf D}^{(k)}\|=0.
\end{align}
\end{proof}

For any ${\bf w}^*$ being the stationary point, there exists subsequence ${\bf w}^{k_j}\rightarrow {\bf w}^*$, with \eqref{vanish}, ${\bf w}^{k_j+1}\rightarrow {\bf w}^*$.

Now, we claim that \eqref{para-re} can be always satisfied. This is because the parameters $\rho_1$ and $\rho_2$ are set by the users. Noting   $c_3=O(\frac{1}{\rho_1^2})$, and $c_4=O(\frac{1}{\rho_1^2})$ (proved in Lemma 4),
$$\lim_{\rho_1\rightarrow+\infty,\rho_2\rightarrow 0}\rho_1c_3+\rho_2c_1+\rho_2c_2+\rho_1c_4=0.$$
Thus, for any $\beta>0$, we can choose enough large $\rho_1$ and enough small $\rho_2$ such that \eqref{para-re} is satisfied.

\section{Experimental Evaluation}\label{experimentalresults}
\subsection{Experimental Setup}
All experiments are conducted on a Windows machine with an Intel i7-7700K CPU (@4.20GHz) and 16GB memory. Binary classification tasks are performed on five LIBSVM benchmark datasets\footnote{https://www.csie.ntu.edu.tw/\~\ cjlin/libsvmtools/datasets/}: heart\_scale, mushrooms, real\_sim, news20 and rcv1.binary. These datasets are summarized in Table \ref{tab:data}.
Heart\_scale and mushrooms are small-scale datasets with small number of samples and low dimension; while real\_sim,  news20 and rcv1.binary are large-scale and high dimensional datasets. Another important feature of the last three large-scale datasets is that they enjoy very sparse structures. All datasets (except rcv1.binary) are split into a train set and a test set with 9:1 via stratified selection.

\begin{table}[!htbp]
	\renewcommand{\arraystretch}{1.2}
	\centering
	\caption{Summary of five real-world LIBSVM datasets.}
	\label{tab:data}
	\begin{tabular}{>{\centering}p{1.8cm}|>{\centering}p{1.5cm}>{\centering}p{1.5cm}>{\centering}p{1.5cm}}
		\hline
		\textbf{dataset} & \textbf{\#samples} & \textbf{\#features} & \textbf{sparsity}
		\tabularnewline
		\hline\hline
		heart\_scale & 270 & 13 & 96.24\%
		\tabularnewline
		mushrooms & 8124 & 112 &18.75\%
		\tabularnewline
		real\_sim & 72309 & 20958 &0.24\%
		\tabularnewline
		news20 & 19996 & 1355191 & 0.03\%
		\tabularnewline
		rcv1.binary & 697641 & 47236 &0.16\%
		\tabularnewline
		\hline
	\end{tabular}
\end{table}

We report the experimental results with the SCAD- and MCP-penalized SVMs.
Following methods are included in our comparison:
\begin{itemize}
\item the successive quadratic algorithm for the SCAD-penalized hinge loss function (SCAD SVM\footnote{https://faculty.franklin.uga.edu/cpark/content/software-packages}) \cite{zhang2005gene}.
\item the reweighted  $\ell_1$ scheme for the MCP-penalized squared hinge loss function (RankSVM-MCP\footnote{http://remi.flamary.com/soft/soft-ranksvm-nc.html}) \cite{laporte2014nonconvex}.
\item the generative shrinkage and thresholding (GIST\footnote{http://www.public.asu.edu/~jye02/Software/GIST/}) algorithm \cite{gong2013a}. Note that this algorithm minimizes the SCAD- and MCP-penalized squared hinge loss functions here.
\item the proposed fast and efficient ADMM algorithm (FEADMM\footnote{FEADMM includes a piece of modified C code of the GIST software for performing the update of ${\bf z}$.}) to SCAD- and MCP-penalized hinge loss functions.
\end{itemize}

The FEADMM and GIST are implemented in Matlab plus C; SCAD SVM and RankSVM-MCP are implemented in Matlab. It is worth noting that SCAD SVM is designed for the SCAD-penalized SVMs; and RankSVM-MCP covers the MCP regularizer instead of the SCAD regularizer. Therefore, the comparisons can be divided into two groups: 1) FEADMM, GIST and SCAD SVM with SCAD-penalized SVMs; 2) FEADMM, GIST and RankSVM-MCP with MCP-penalized SVMs. For performance metrics, we evaluate all methods by measuring the running time, the number of iterations and the prediction accuracy.

 In terms of parameters setting, zero vectors are chosen as the starting point of ${\bf w}$ for all the evaluated methods (except SCAD SVM). We set $\theta=3.7$ for the SCAD penalty and $\theta=3$ for the MCP penalty as suggested in the literature.  
 The tunning parameter $\lambda$ for GIST, SCAD SVM, RankSVM-MCP is chosen from the set \{$2^{-18}$,\ldots,$2^{4}$\} by five-fold cross validation. For FEADMM, we empirically set ${\lambda=2^{-6}}$; and  $\rho_1$ and $\rho_2$ are chosen by a grid search over \{0.01, 0.1, 1, 1.5, 5, 10\}.

In our experiments, the terminate condition of FEADMM is designed by measuring the change of objective value between consecutive iterations. We define the relative change of the objective value as $\epsilon=|\frac{\textnormal{obj}^{(k+1)}-\textnormal{obj}^{(k)}}{\textnormal{obj}^{(k)}}|$ where $\textnormal{obj}^{(k)}=\frac{1}{n}{\bf 1}^\top{\bm \xi}^{(k)}+P({\bf z}^{(k)})$. FEADMM is terminated when $\epsilon < 10^{-4}$ or the number of iterations exceeds 1000.

\subsection{Simulation and Discussion}
\subsubsection{Comparison with other methods}
We report the experimental results to demonstrate the efficiency of FEADMM. Experimental results on the two small-scale datasets are presented in Fig.\ref{small}; and experimental results on the three large-scale datasets are shown in Fig.\ref{large}. In all the figures, the $x$-axes all denote the CPU time (in seconds); and the $y$-axes all denote the prediction accuracy. Corresponding results are summarized in Table \ref{tab:small} and Table \ref{tab:large}, respectively. Note that SCAD SVM and RankSVM-MCP don't appear in Fig.\ref{large} and Table \ref{tab:large} because we find that they are unable to handle the large-scale datasets. For the evaluation on the small-scale datasets, comparisons of FEADMM with GIST and SCAD on SCAD-penalized SVMs are shown in Fig.\ref{heart_scale_scad} and Fig.\ref{mushrooms_scad}; Comparisons of FEADMM with GIST and RankSVM-MCP on MCP-penalized SVMs are shown in Fig.\ref{heart_scale_mcp} and Fig.\ref{mushrooms_mcp}.

\begin{figure}[!htbp]
	\centering
	\subfigure[heart\_scale SCAD]{\includegraphics[width=0.23\textwidth, height=3.8cm]{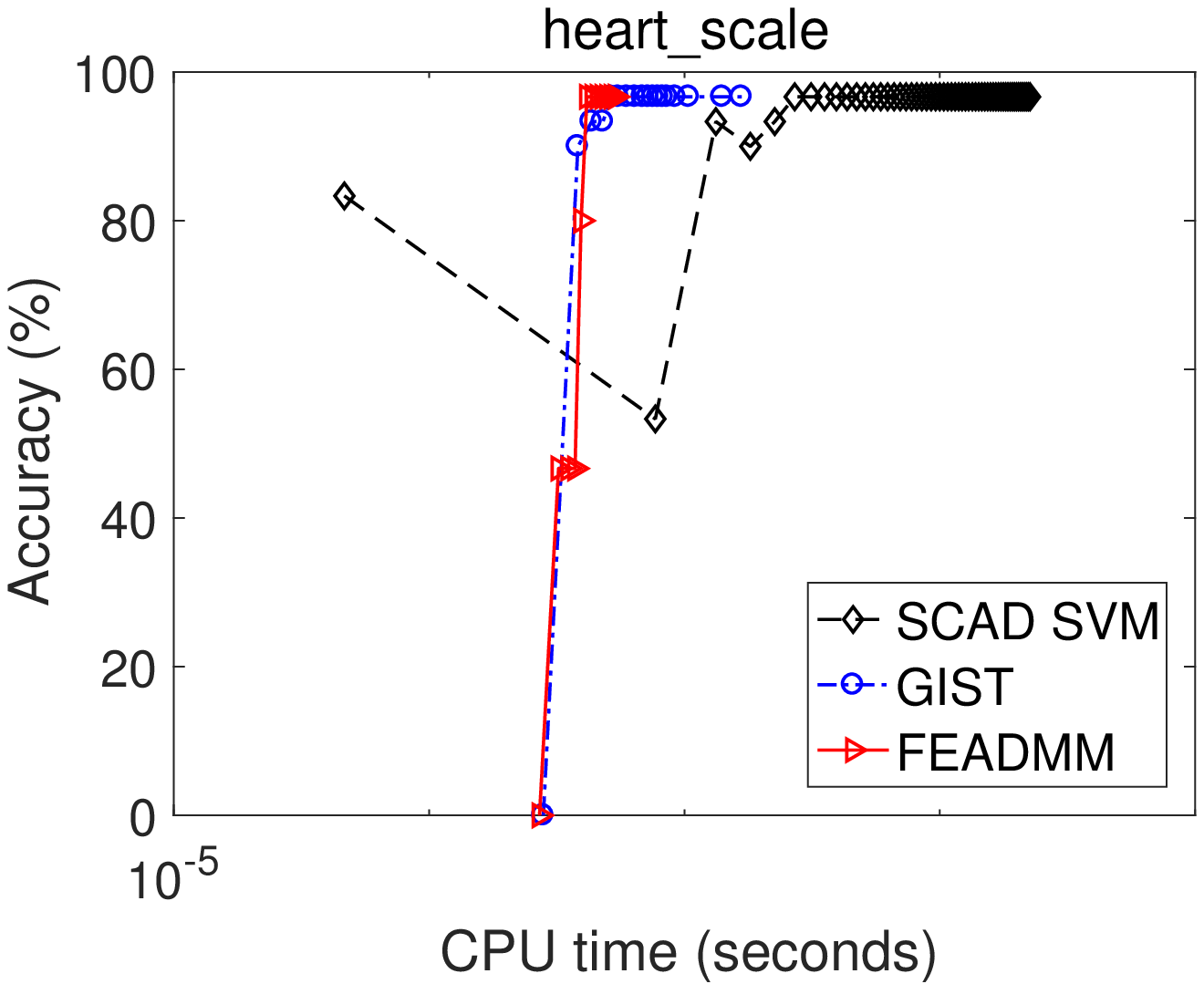}\label{heart_scale_scad}}
    \subfigure[heart\_scale MCP]{\includegraphics[width=0.23\textwidth, height=3.8cm]{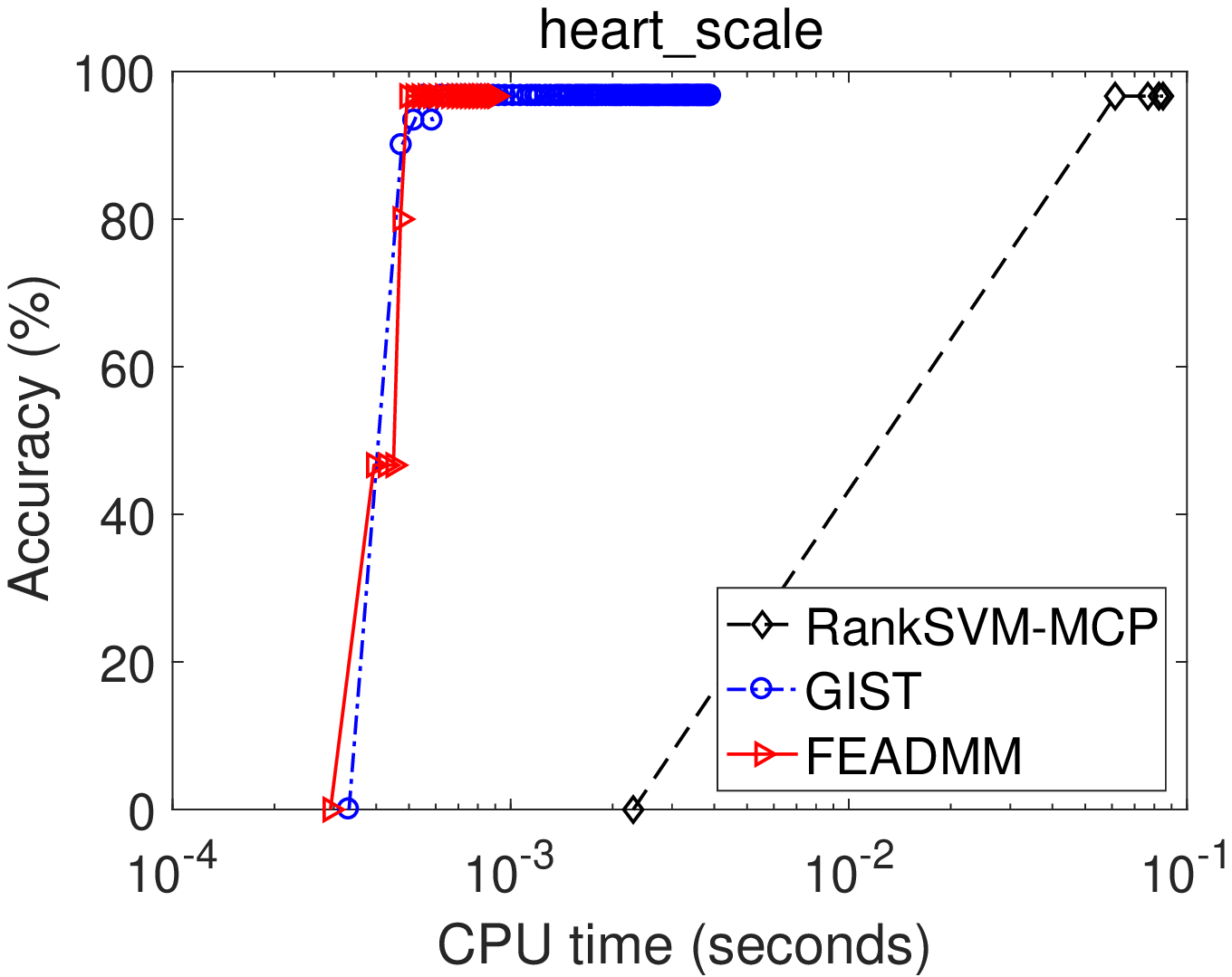}\label{heart_scale_mcp}}
	\subfigure[mushrooms SCAD]{\includegraphics[width=0.23\textwidth, height=3.8cm]{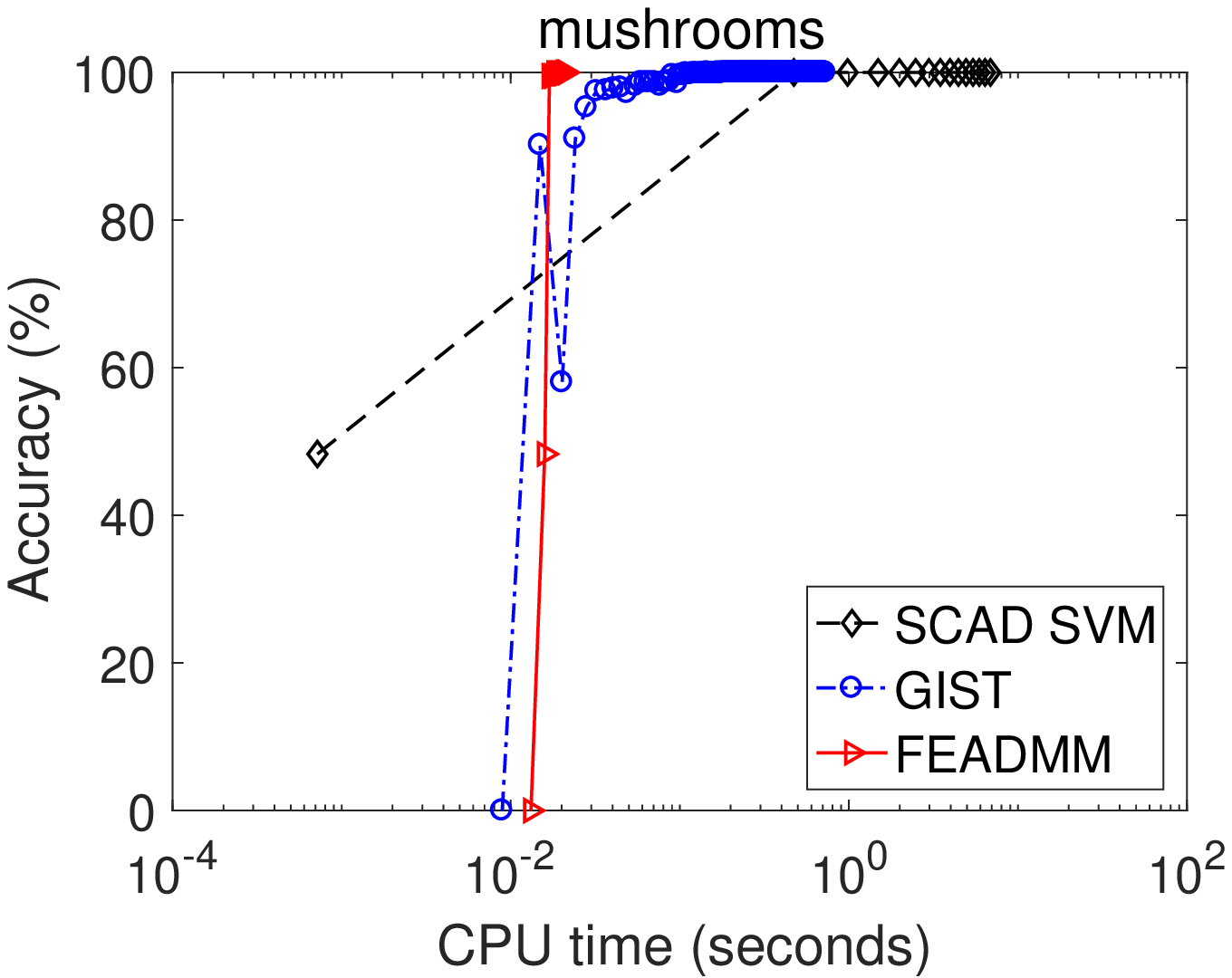}\label{mushrooms_scad}}
	\subfigure[mushrooms MCP]{\includegraphics[width=0.23\textwidth, height=3.8cm]{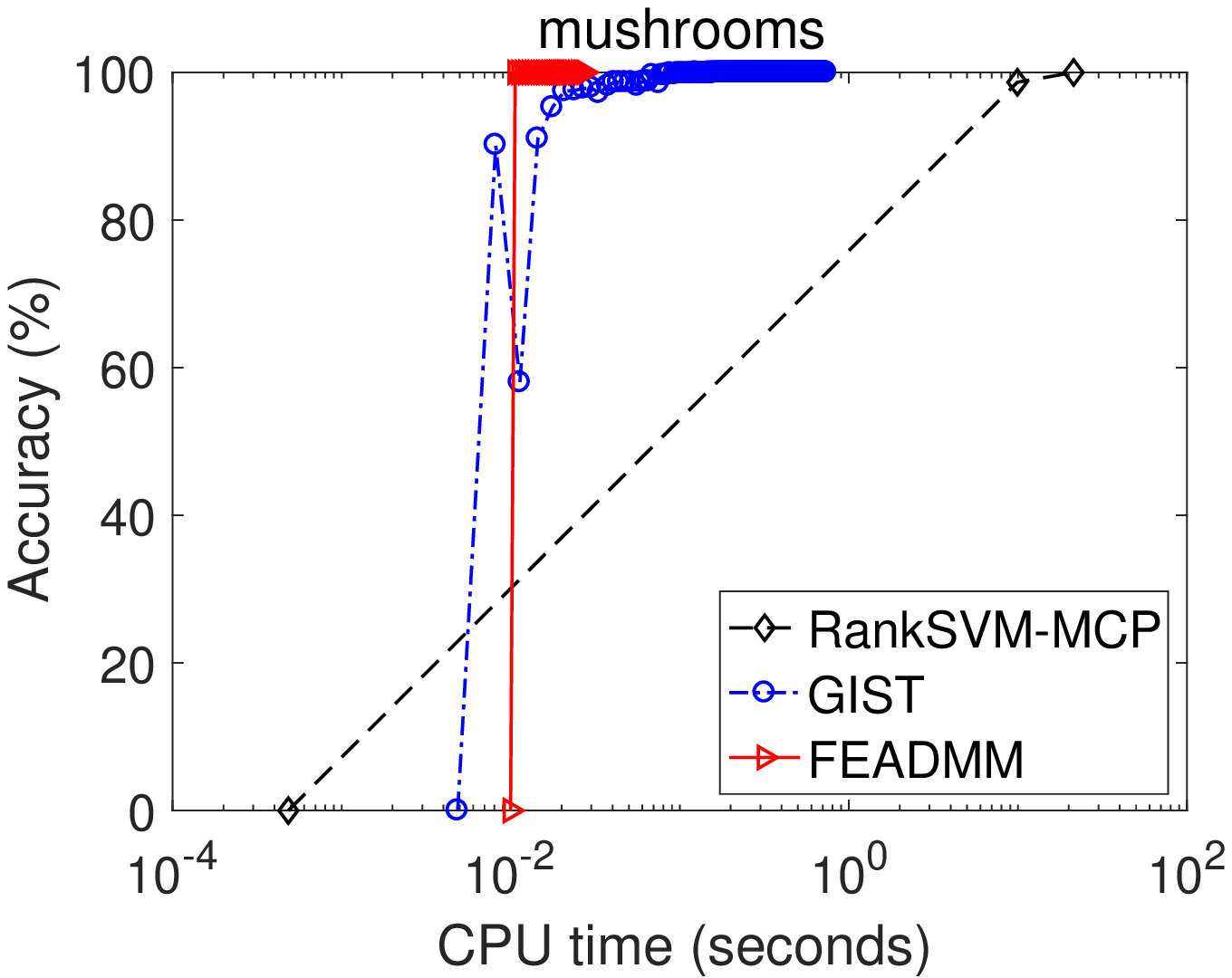}\label{mushrooms_mcp}}
	\caption{Comparison of FEADMM with three existing methods on the small-scale datasets. Prediction accuracy vs CPU time (in seconds) with SCAD penalty and MCP is shown in the left column and right column, respectively. The red solid lines stand for the FEADMM; the blue dashed lines stand for the GIST; the black dashed lines stand for the SCAD SVM in in Fig.\ref{heart_scale_scad} and Fig.\ref{mushrooms_scad} and RandSVM-MCP in Fig.\ref{heart_scale_mcp} and Fig.\ref{mushrooms_mcp}.}
	\label{small}
\end{figure}

\renewcommand{\multirowsetup}{\centering}
{\begin{table*}[!htbp]
		\renewcommand{\arraystretch}{1.2}
		\centering
		\caption{Comparison of FEADMM with three existing methods for the SCAD- and MCP-penalized SVMs on the small-scale datasets. The best results are highlighted in boldface.}
		\label{tab:small}
		\begin{tabular}
			 {>{\centering}p{1.5cm}>{\centering}p{1cm}>{\centering}p{2cm}|>{\centering}p{1.5cm}>{\centering}p{2.5cm}>{\centering}p{2cm}>{\centering}p{1.5cm}>{\centering}p{1.5cm}}
			\hline
			\textbf{dataset} & \textbf{penalty}& \textbf{method} & \textbf{\#iteration} &  \textbf{pre-computation time}  & \textbf{iteration time}& \textbf{running time} & \textbf{accuracy}
			\tabularnewline
			\hline
			\hline
			\multirow{6}{1.5cm}{heart\_scale}&\multirow{3}{1cm}{SCAD}
			&SCAD SVM  & 58 & $\approx$ 0 & 0.023 & 0.023  &\textbf{96.67}\%
			\tabularnewline
			& &GIST  &18 & 2.80e-4 & 0.001 & 0.002  &\textbf{96.67}\%
			\tabularnewline
			& &FEADMM  &\textbf{12}& 2.71e-4 & 2.88e-4 &\textbf{5.59e-4}  & \textbf{96.67}\%
			\tabularnewline
			\cline{2-8}
			&\multirow{3}{1cm}{MCP}
			&RankSVM-MCP  & \textbf{4} & $\approx$ 0 & 0.024  & 0.024 &\textbf{96.67}\%
			\tabularnewline
			& &GIST  &87 & 3.25e-4& 0.004 &  0.004 &\textbf{96.67}\%
			\tabularnewline
			&	&FEADMM  &24 &2.93e-4  & 6.27e-4 &\textbf{0.001} & \textbf{96.67}\%
			\tabularnewline
			\hline
			\multirow{6}{1.5cm}{mushrooms}&\multirow{3}{1cm}{SCAD}
			& SCAD SVM & 15& $\approx$ 0 & 6.879 & 6.879  & \textbf{100}\%
			\tabularnewline
			& & GIST  &194 & 0.009 & 0.718&0.726  &\textbf{100}\%
			\tabularnewline
			& & FEADMM  &\textbf{11} & 0.013 & 0.008&\textbf{0.022}  &\textbf{100}\%
			\tabularnewline
			\cline{2-8}
			&\multirow{3}{1cm}{MCP}
			& RankSVM-MCP & \textbf{2} &$\approx$ 0 &21.39 & 21.39  & \textbf{100}\%
			\tabularnewline
			& & GIST  &205 & 0.005& 0.730 &0.735  &\textbf{100}\%
			\tabularnewline
			& & FEADMM  &28 & 0.01& 0.02 &\textbf{0.03} &\textbf{100}\%
			\tabularnewline
			\hline
		\end{tabular}
	\end{table*}
}

\textbf{Running time and convergence.}
From the observation of Fig.\ref{small} and Fig.\ref{large}, we can reach to following conclusions: First, it is clear that FEADMM runs fast and can always converge within only a few tens of iterations for any dataset we evaluated. Second, with SCAD-penalized SVMs, SCAD SVM is inferior to FEADMM and GIST in terms of running time.
Third, with MCP-penalized SVMs, RankSVM-MCP performs worst in terms of total running time despite that it needs the minimum number of iterations. Fourth, both SCAD SVM and RankSVM-MCP are limited to the processing of small-scale datasets.
Fifth, comparing FEADMM with GIST, we can see that the number of iterations of FEADMM is consistently much less than that of GIST. In the aspect of running time, FEADMM outperforms GIST in most cases. Iteration numbers and running time statistics in Table \ref{tab:small} and Table \ref{tab:large} validate this point. For the evalutations on small-scale datasets, FEADMM takes much less running time than GIST. For evaluations on large-scale datasets, FEADMM only takes more running time than GIST on the rcv1.binary dataset whose training samples is much larger than its dimension. However, FEASMM runs faster than GIST on the news20 datset with extremely high dimension. This demonstrates the superiority of FEADMM in handling high dimensional datasets.



\textbf{Prediction accuracy.} Fig.\ref{small} and Fig.\ref{large} show that the prediction accuracy of each method increases along with the CPU time. Specially, in all the figures, the trend of the red solid lines is always almost a straight line up. This again demonstrates the fast convergence rate of FEADMM; FEADMM can quickly attain a high prediction accuracy.
In addition, Table \ref{tab:small} shows that the four evaluated methods attain the same prediction accuracy, which demonstrates that these methods are comparable in terms of prediction accuracy on the small-scale datasets. On the other hand, Table \ref{tab:large} shows that FEADMM performs slightly better than GIST on the large-scale datasets in the aspect of prediction accuracy. Moreover, from Table \ref{tab:small} and \ref{tab:large} we find that FEADMM attains comparable prediction accuracy with SCAD- and MCP-penalized SVMs. The discussions above demonstrate that FEADMM enjoys fast execution speed as well as strong generalization ability when solving the SCAD- and MCP-penalized SVMs.

\renewcommand{\multirowsetup}{\centering}
{\begin{table*}[!htbp]
		\renewcommand{\arraystretch}{1.2}
		\centering
		\caption{Comparison of FEADMM with GIST for SCAD- and MCP-penalized SVMs on the large-scale datasets. The best results are highlighted in boldface.}
		\label{tab:large}
		\begin{tabular}
			 {>{\centering}p{1.5cm}>{\centering}p{1cm}>{\centering}p{2cm}|>{\centering}p{1.5cm}>{\centering}p{2.5cm}>{\centering}p{2cm}>{\centering}p{1.5cm}>{\centering}p{1.5cm}}
			\hline
			\textbf{dataset} & \textbf{penalty}& \textbf{method} & \textbf{\#iteration} &  \textbf{pre-computation time}  & \textbf{iteration time}& \textbf{running time} & \textbf{accuracy}
			\tabularnewline
			\hline
			\hline
			\multirow{4}{1.5cm}{real\_sim}&
			\multirow{2}{1cm}{SCAD}
			&  GIST  &1000+& 0.13& 120.51 & 120.64  &96.17\%
			\tabularnewline
			& &  FEADMM  &\textbf{35} & 28.00 & 14.17 &\textbf{42.17}  & \textbf{97.42}\%
			\tabularnewline
			\cline{2-8}
			&\multirow{2}{1cm}{MCP}
			&  GIST  &1000+ & 0.35& 119.33 &119.68& 96.17\%
			\tabularnewline
			& &  FEADMM  &\textbf{12} &27.78 &5.46 &\textbf{33.24}  & \textbf{97.40}\%
			\tabularnewline
			\hline
			\multirow{4}{1.5cm}{rcv1.binary} &
			\multirow{2}{1cm}{SCAD}
			& GIST  &1000+ & 0.07 & 49.49& \textbf{49.56}  &95.08\%
			\tabularnewline
			& & FEADMM  &\textbf{27}& 50.52 & 9.68  &60.20 & \textbf{96.02}\%
			\tabularnewline
			\cline{2-8}
			&\multirow{2}{1cm}{MCP}
			& GIST  &1000+ & 0.05&48.17 & \textbf{48.22}  &94.95\%
			\tabularnewline
			& & FEADMM  &\textbf{11} & 51.49 & 3.89&55.38 & \textbf{96.02}\%
			\tabularnewline
			\hline
			\multirow{4}{1.5cm}{news20} &
			\multirow{2}{1cm}{SCAD}
			& GIST  &209 & 0.31 & 67.07& 67.38  & 95.59\%
			\tabularnewline
			& & FEADMM  &\textbf{16} & 52.62 & 6.09&\textbf{58.71} & \textbf{95.84}\%
			\tabularnewline
			\cline{2-8}
			&\multirow{2}{1cm}{MCP}
			& GIST  &225 & 0.29 &73.65 & 73.94 & 95.94\%
			\tabularnewline
			& & FEADMM  &\textbf{23}  & 52.60& 8.69 &\textbf{61.29}& \textbf{96.19}\%
			\tabularnewline
			
			\hline
		\end{tabular}
	\end{table*}
}
\begin{figure}[!htbp]
	\centering
	\subfigure[real\_sim SCAD]{\includegraphics[width=0.23\textwidth, height=3.8cm]{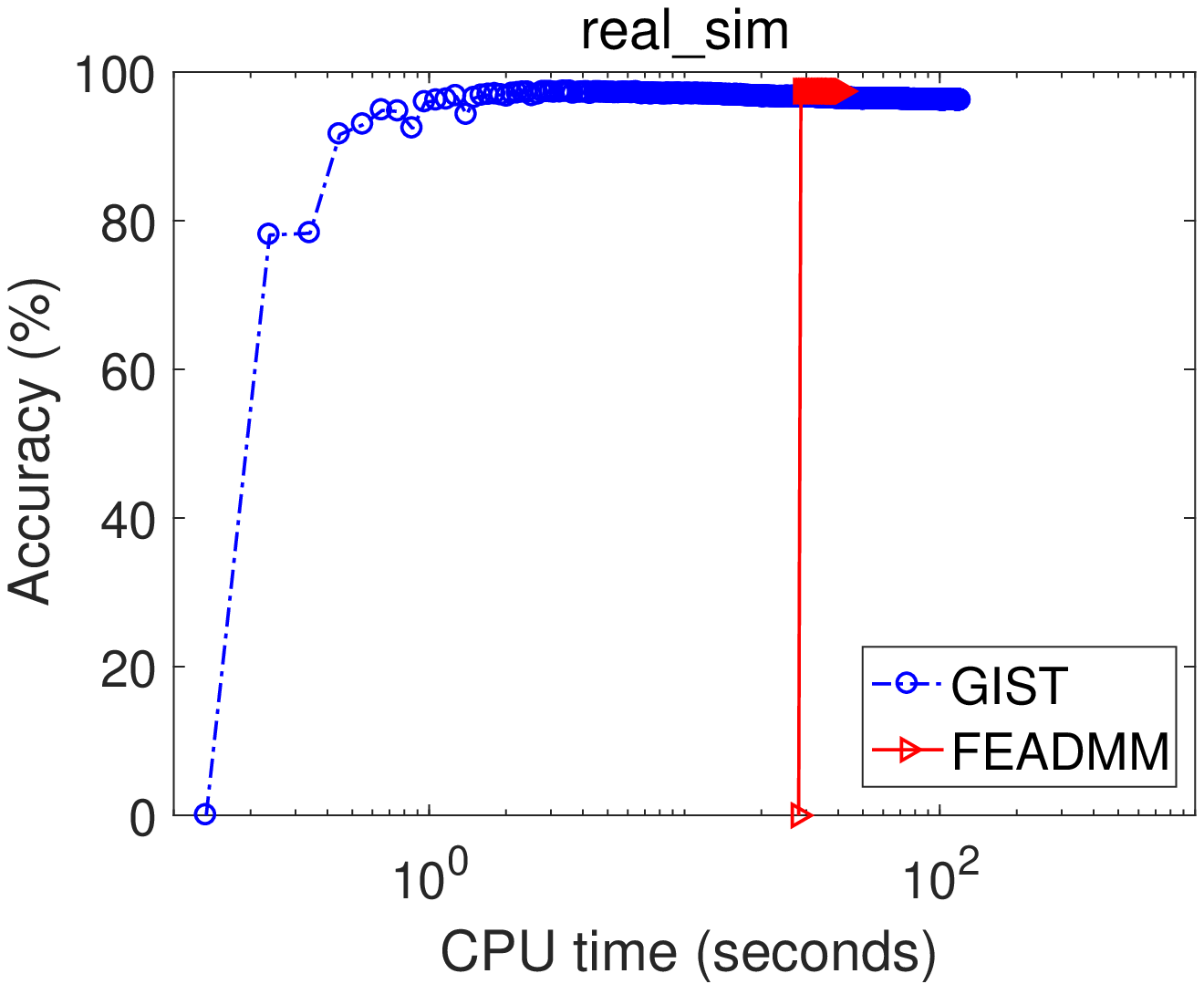}\label{real_sim_scad}}
    \subfigure[real\_sim MCP]{\includegraphics[width=0.23\textwidth, height=3.8cm]{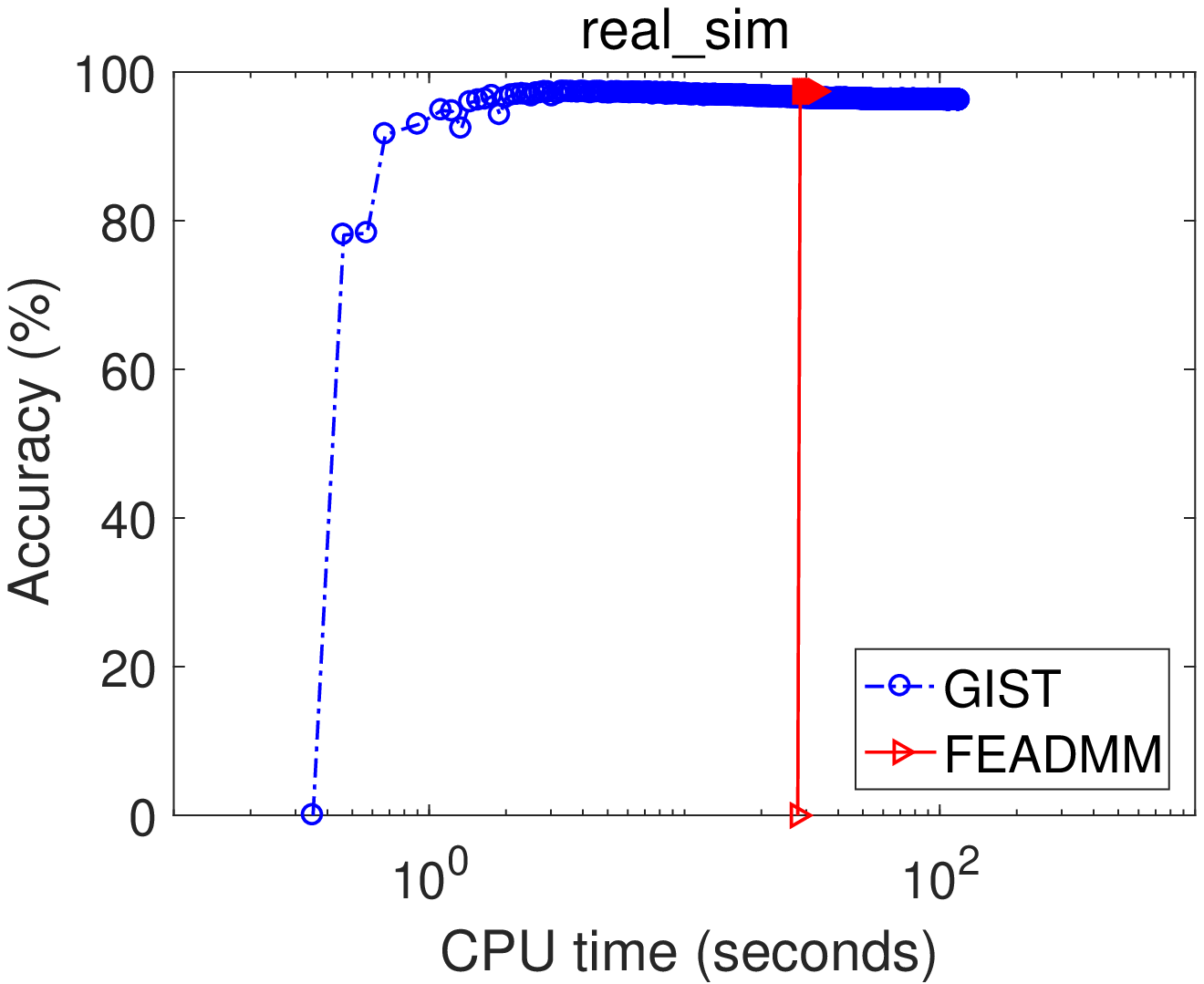}\label{real_sim_mcp}}
    \subfigure[news20 SCAD]{\includegraphics[width=0.23\textwidth, height=3.8cm]{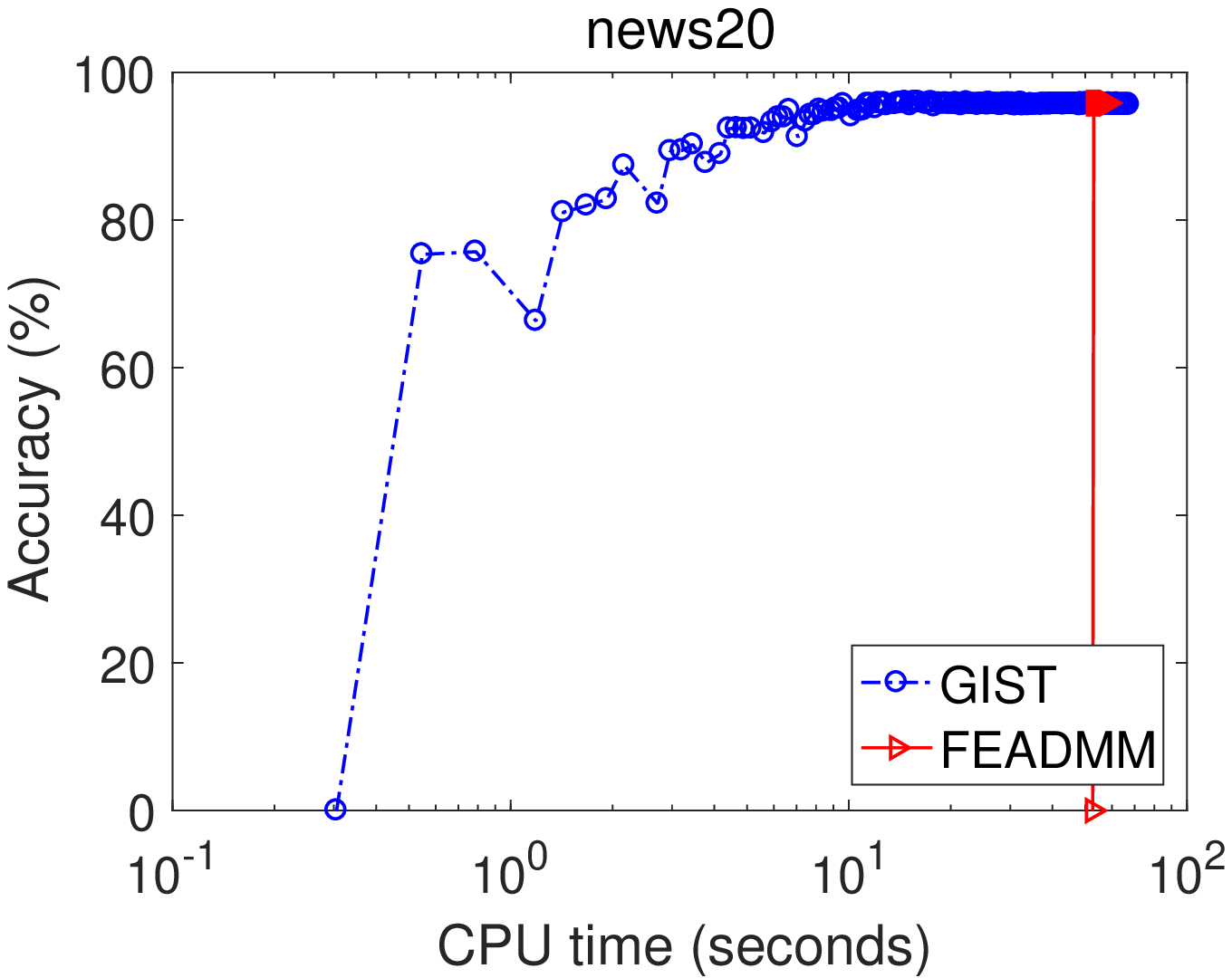}\label{news20_scad}}
	\subfigure[news20 MCP]{\includegraphics[width=0.23\textwidth, height=3.8cm]{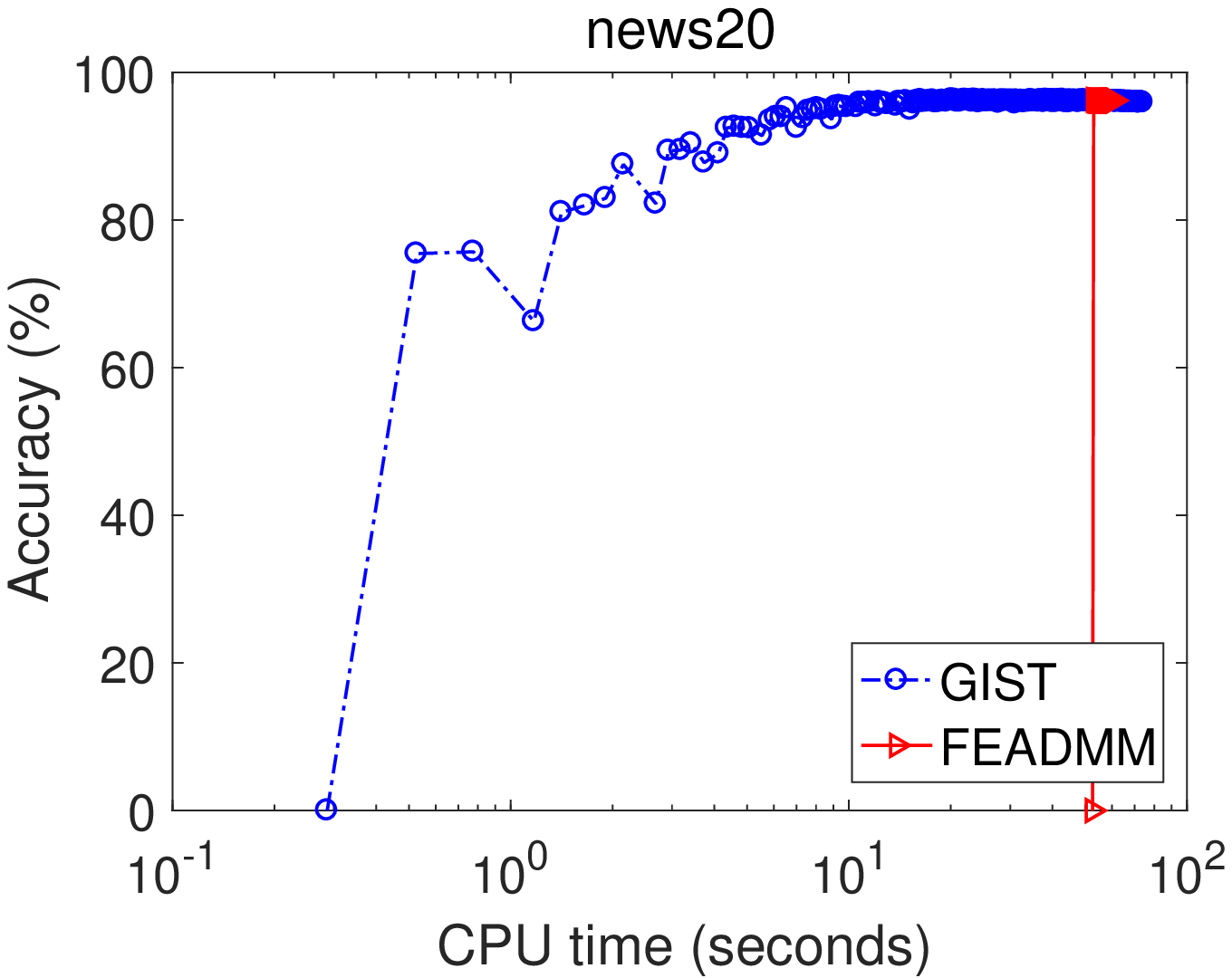}\label{news20_mcp}}
	\subfigure[rcv1.binary SCAD]{\includegraphics[width=0.23\textwidth, height=3.8cm]{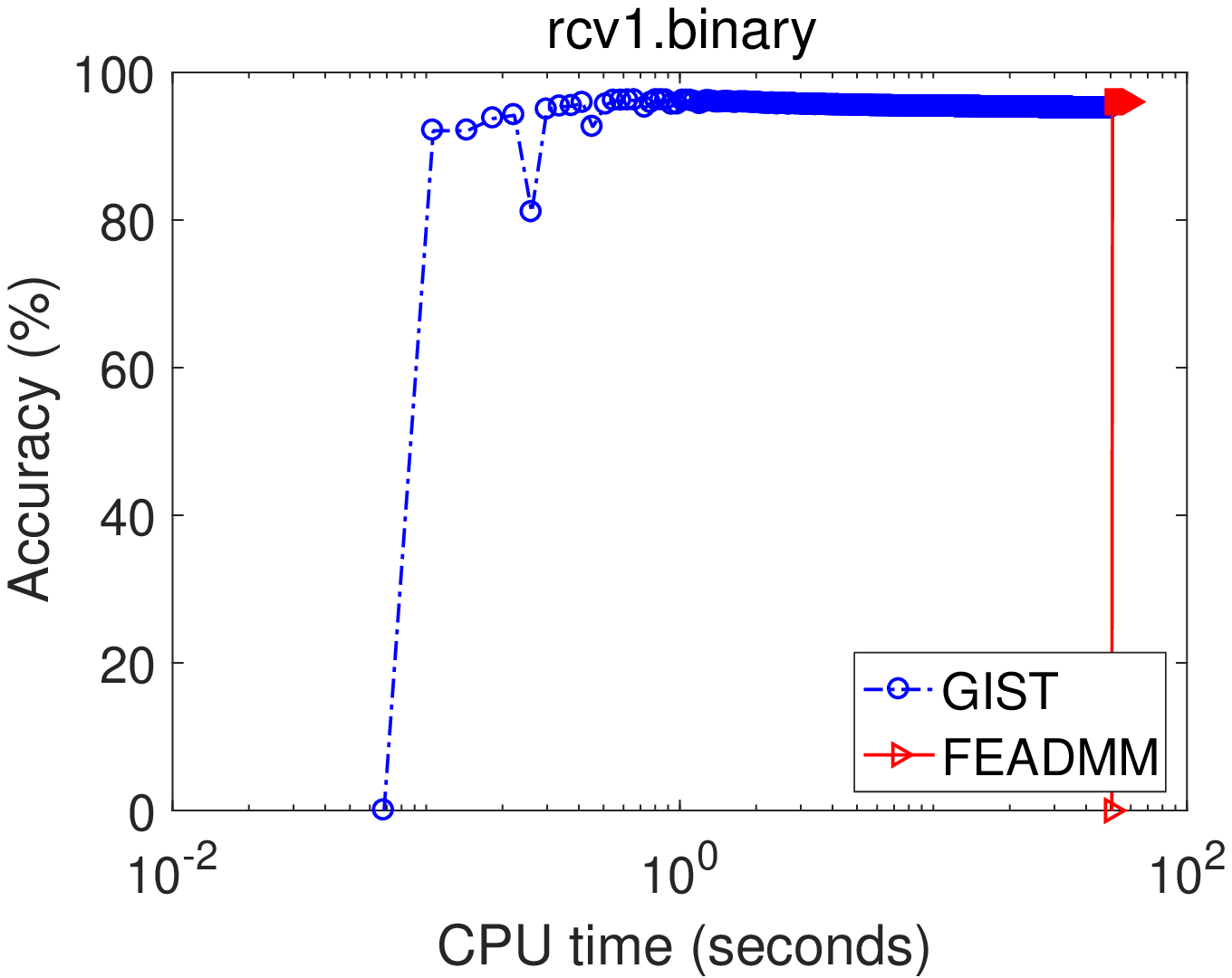}\label{rcv1_scad}}
    \subfigure[rcv1.binary MCP]{\includegraphics[width=0.23\textwidth, height=3.8cm]{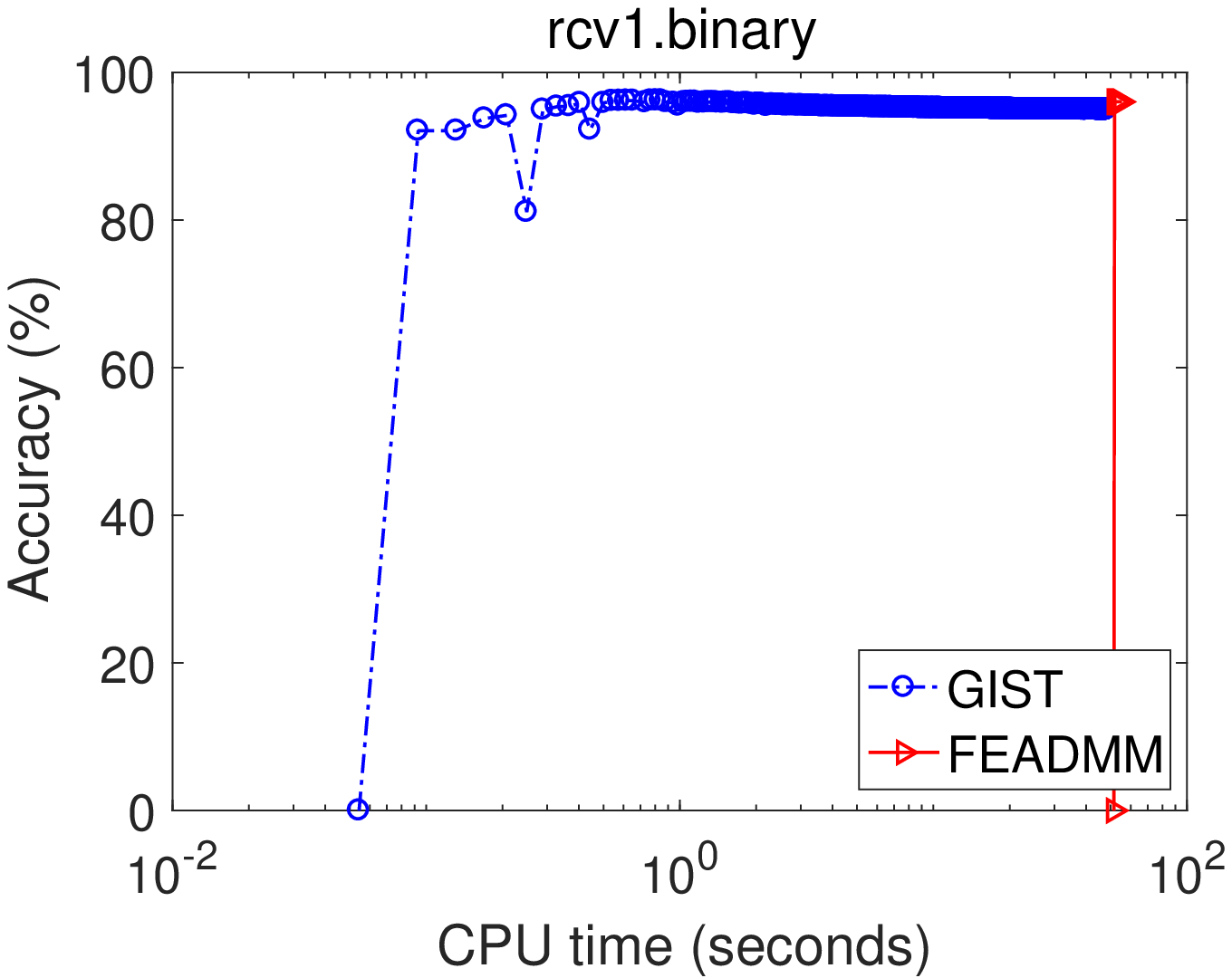}\label{rcv1_mcp}}
	\caption{Comparison of FEADMM with GIST for SCAD- and MCP-penalized SVMs on the large-scale datasets. Prediction accuracy vs CPU time (in seconds) with SCAD penalty and MCP is shown in the left column and right column, respectively. The red solid lines stand for the FEADMM; the blue dashed lines stand for the GIST.}
	\label{large}
\end{figure}

\subsubsection{Computational burden analysis} Based on the observations of Fig.\ref{small} and Fig.\ref{large}, we find that in all the figures it always takes some time before the curve of prediction accuracy of each method begins to go up. In fact, all the evaluated methods need to do some pre-computations before the iteration starts. In order to analyze the main computational burden of each method, we split the total running time of each method into two parts: the pre-computation time costing at the pre-computation stage and the iteration time costing at the iteration stage. Table \ref{tab:small} and Table \ref{tab:large} summarize the pre-computation and iteration time of each method on the small-scale and large-scale datasets, respectively.
From Table \ref{tab:small}, we see that SCAD SVM and RankSVM-MCP spend litte time at the pre-compuatation stage. The pre-computation time of GIST is much shorter than its iteration time. In contrast, the pre-computation time of FEADMM is almost close to its iteration time. This means that when evaluated on the small-scale datasets, the main computational burden of the three existing methods lies in their iteration stages; while for FEADMM, both the pre-computation and iteration procedure play an important role in the total running time. Comparing Table \ref{tab:large} with Table \ref{tab:small}, we see that both the pre-computation time of FEADMM and GIST increases when evaluated on large-scale datasets. Meanwhile, Table \ref{tab:large} illustrates that GIST spends much less time at the pre-compuatation stage yet. In contrast, we find that the pre-computation time of FEADMM exceeds its iteration time a lot. Therefore, it is clear that the main burden of FEADMM lies in the pre-computation stage, which occupies a large percentage on the large-scale datasets. This result verifies the computational analysis in Section \ref{algorimanalysis}.



\section{Conclusions}\label{conclusions}
In order to solve the nonconvex penalized SVMs, this paper proposed an efficient algorithm based on the framework of ADMM.  We design a novel mechanism that updates ${\bf w}$ according to the values of the number of training data ($n$) and the dimension of the training data ($d$), which gives rise to much lower computational cost.
Moreover, the burden of the algorithm has been transferred to the outside of the iterations.
We detailedly studied the computational complexity and the convergence of the proposed algorithm.
The extensive experimental evaluations demonstrate that the proposed algorithm outperforms other three state-of-the-art methods
in terms of running time and prediction accuracy. In special, this paper actually proposes a general framework to SVMs with sparsity-inducing regularizations. SVMs with other sparsity-inducing regularizations can be efficiently solved by applying the proposed algorithm as long as Equation (\ref{eq14}) admits a solution. For future work, we will further explore to incorporate other sparsity-including penalties, like $\ell_p$ penalty with $0<p<1$ and the elastic net penalty \cite{Ye2011Efficient,liu2013iterative}, into the proposed framework.

\begin{appendix}
	\section{Closed-form Solution of problem (\ref{eq14})}
	\label{zupdate}
	Here we present the closed-form solution of problem (\ref{scadbefore}) for LSP, SCAD penalty, MCP and capped  $\ell_1$ penalty. All results are obtained by applying the conclusions drawn in \cite{gong2013a}. Here, $\psi_{i}^{(k)}$ indicates the $i$th entry of vector ${\bm \psi}$ in iteration $k$.
	
	\begin{enumerate}[(I)]
		\item \textbf{LSP:}
		$z_i^{(k+1)}=\text{sign}(\psi_i^{(k+1)}x)$, where
		$x=\argmin_{z_i\in\mathcal{C}} \  \frac{1}{2}(z_i-|\psi_i^{(k+1)}|)^2+\frac{\lambda}{\rho_1}\text{log}(1+z_i/\theta)$ and $\mathcal{C}$ is a set composed of $3$ elements or $1$ element.
		
		If $\rho_1^2(|\psi_i^{(k+1)}|-\theta)^2-4\rho_1(\lambda-\rho_1|\psi_i^{(k+1)}|\theta)\gneq0$,
		
		$\mathcal{C}=\{0, \\
		 \lbrack\frac{\rho_1(|\psi_i^{(k+1)}|-\theta)+\sqrt{\rho_1^2(|\psi_i^{(k+1)}|-\theta)^2-4\rho_1(\lambda-\rho_1|\psi_i^{(k+1)}|\theta)}}{2\rho_1}\rbrack_+ \\
		 \lbrack\frac{\rho_1(|\psi_i^{(k+1)}|-\theta)-\sqrt{\rho_1^2(|\psi_i^{(k+1)}|-\theta)^2-4\rho_1(\lambda-\rho_1|\psi_i^{(k+1)}|\theta)}}{2\rho_1}\rbrack_+
		\}$. \\
		Otherwise, 	$\mathcal{C}=\{0\}$.
		
		\item \textbf{SCAD:}
		Consider that $\theta>2$ and let $x_1=
		\text{sign}(\psi_i^{(k+1)})\min(\lambda, \max(0,|\psi_i^{(k+1)}|-\lambda/\rho_1))\ \ s.t. \ \ |z_i^{(k)}|\leq\lambda$,
		$
		 x_2=\text{sign}(\psi_i^{(k+1)})\min(\theta\lambda,\max(\lambda,\frac{\rho_1|\psi_i^{(k+1)}|(\theta-1)-\theta\lambda}{\rho_1(\theta-2)})  \ \ s.t. \\  \lambda<|z_i^{(k)}|\leq\theta\lambda
		$,
		$
		x_3=\text{sign}(\psi_i^{(k+1)})\max(\theta\lambda,|\psi_i^{(k+1)}|) \ \ s.t. \ \ |z_i^{(k)}|>\theta\lambda
		$.
		Thus we have
		$z_i^{(k+1)}=\argmin_{m}h_i(m) \ \ s.t. \ \ m\in\{x_1,x_2,x_3\}$, where $h_i(m)=\frac{1}{2}(m-\psi_i^{(k+1)})^2+\frac{1}{\rho_1}p_\lambda(m)
		$ and $p_\lambda(m)$ refers to the SCAD regularizer in Table \ref{Tab:Penalty}.

		\item \textbf{MCP:}
		Let $x_1=\text{sign}(\psi_i^{(k+1)})m$ and $x_2=\text{sign}(\psi_i^{(k+1)})\max(\theta\lambda,|\psi_i^{(k+1)}|)$ where $m=\argmin_{z_i\in\mathcal{C}}\frac{1}{2}(z_i-|\psi_i^{(k+1)}|)^2+\frac{\lambda}{\rho_1}z_i-\frac{z_i^2}{2\theta}$.
		Here, $\mathcal{C}=\{0,\theta\lambda,\min(\theta\lambda,\max(0,\frac{\theta(\rho_1|\psi_i^{(k+1)}|-\lambda)}{\rho_1(\theta-1)}))
		\}$, if $\theta-1\neq0$, and $\mathcal{C}=\{0,\theta\lambda\}$ otherwise.
		\\
		Then we have
		$z_i^{(k+1)}$ $=\left\{\begin{array}{ll}
		x_1, & \text{if} \ h_i(x_1)\leq h_i(x_2), \\
		x_2, & \text{otherwise}. \\
		\end{array}\right.$
		Here
		$h_i(m)=\frac{1}{2}(m-\psi_i^{(k+1)})^2+\frac{1}{\rho_1}p_\lambda(m)$ and $p_\lambda(m)$ refers to the MCP regularizer in Table \ref{Tab:Penalty}.
		
		\item \textbf{Capped $\ell_1$:}
		Let $x_1=\text{sign}(\psi_{i}^{(k+1)})\max(\theta,|\psi_{i}^{(k+1)}|)$ \ \ s.t. \ \ $|z_i^{(k)}| \geq \theta$,
		$x_2=\text{sign}(\psi_{i}^{(k+1)})\min(\theta,\max(0,|\psi_{i}^{(k+1)}|-\lambda/\rho_1))$ \ \ s.t. \ \ $|z_i^{(k)}| \leq \theta$.
		\\  Then we have $z_i^{(k+1)}$ $=\left\{\begin{array}{ll}
		x_1, & \text{if} \ h_i(x_1)\leq h_i(x_2), \\
		x_2, & \text{otherwise}. \\
		\end{array}\right.$ Here
		$h_i(m)=\frac{1}{2}(m-\psi_{i}^{(k+1)})^2+\frac{1}{\rho_1}p_\lambda(m)$ and $p_\lambda(m)$ indicates the Capped $\ell_1$ regularizer in Table \ref{Tab:Penalty}.

	\end{enumerate}
	
\end{appendix}


\end{document}